\documentclass[runningheads]{llncs}

\usepackage[mobile]{eccv}

\usepackage{eccvabbrv}

\usepackage{graphicx}
\usepackage{booktabs}

\usepackage[accsupp]{axessibility}  

\usepackage[utf8]{inputenc}
\usepackage{geometry}
\usepackage{amsmath}
\usepackage{amssymb}
\usepackage{booktabs}
\usepackage{tabularx}
\usepackage[table]{xcolor}
\usepackage{bbm}
\usepackage{dsfont}
\usepackage{multirow}
\usepackage{graphicx}
\usepackage{makecell}
\usepackage{pifont}

\usepackage{amssymb}
\usepackage{graphicx}

\newcommand{\std}[1]{\,{\raisebox{0pt}{\scalebox{0.7}{$\pm #1$}}}}
\newcommand{\cmark}{\scalebox{1.2}{\ding{51}}}
\newcommand{\xmark}{\scalebox{1.2}{\ding{55}}}
\usepackage{threeparttable}

\usepackage{mathtools}          
\newtheorem{assumption}{Assumption}
\newcommand{\UI}{\mathrm{UI}}
\newcommand{\RI}{\mathrm{RI}}
\newcommand{\SI}{\mathrm{SI}}
\newcommand{\drop}{\mathrm{drop}}
\newcommand{\swap}{\mathrm{swap}}
\newcommand{\indep}{\mathrel{\perp\!\!\!\perp}}
\newcommand{\bbot}{\bot}

\usepackage{hyperref}

\usepackage{orcidlink}

\begin{document}

\title{BrainFIBRE: A Foundation Model via Information Decomposition for Brain Microstructure} 

\titlerunning{BrainFIBRE}

\author{Zijian Dong* \and
Yi Lin* \and
Fang Ji* \and Jianxiong Zhou \and Kwun Kei Ng \and Juan Helen Zhou$^{\dagger}$}

\authorrunning{Z. Dong et al.}

\institute{National University of Singapore, Singapore \\
$^{*}$Equal contribution \quad $^{\dagger}$Corresponding author  \\
\email{\{zijian\_dong, helen.zhou\}@nus.edu.sg}, \ 
\email{lin\_yi@u.nus.edu}}

\maketitle

\begin{abstract}
Diffusion MRI is widely used to probe brain microstructure, with particular sensitivity to early cerebrovascular and neurodegenerative changes. Neurite Orientation Dispersion and Density Imaging (NODDI) decomposes the diffusion signal into three biophysically interpretable maps — neurite density index (NDI), orientation dispersion index (ODI), and free water fraction (FWF) — capturing neurite packing, fiber coherence, and extracellular fluid, respectively. These 3D maps provide a rich substrate for learning transferable microstructural representations for the human brain, yet effectively integrating them remains an open challenge: standard representation learning struggles to disentangle the unique information carried by each of the three brain microstructure maps (\emph{i.e.,} NDI, ODI, and FWF)  from their shared and synergistic interactions. Here, we present \textbf{BrainFIBRE} (\textbf{Brain} \textbf{F}oundation Model via \textbf{I}nformation Decomposition for \textbf{BR}ain Microstructur\textbf{E}), the first foundation model for brain microstructure, pretrained on three NODDI-derived microstructure maps from the UK Biobank dataset (55,592 participants). To achieve this, we propose Self-supervised Partial Information Decomposition (SPID), which extends PID-guided multimodal learning to the self-supervised regime for the first time. A novel Counterfactual Candidate Construction (CCC) paradigm perturbs inter-modality alignment through modality dropping and swapping, providing the contrastive signal for a Mixture-of-Experts (MoE) architecture to disentangle unique, synergistic, and redundant information without any downstream label. Evaluated on both Caucasian and Asian cohorts, our model achieves state-of-the-art performance across diverse downstream tasks predicting age, sex, cerebrovascular disease (CeVD) and neurodegenerative markers, and cognitive performance, while yielding neurobiologically interpretable representations that reveal task- and cohort-specific interaction patterns among microstructural compartments. BrainFIBRE establishes a versatile foundation for neuroimaging analysis at the microstructural level. Code is available at \href{https://github.com/hzlab/BrainFIBRE}{https://github.com/hzlab/BrainFIBRE}
\end{abstract}

\section{Introduction}
\label{sec:intro}

\begin{figure}
    \centering
    \includegraphics[width=1\linewidth]{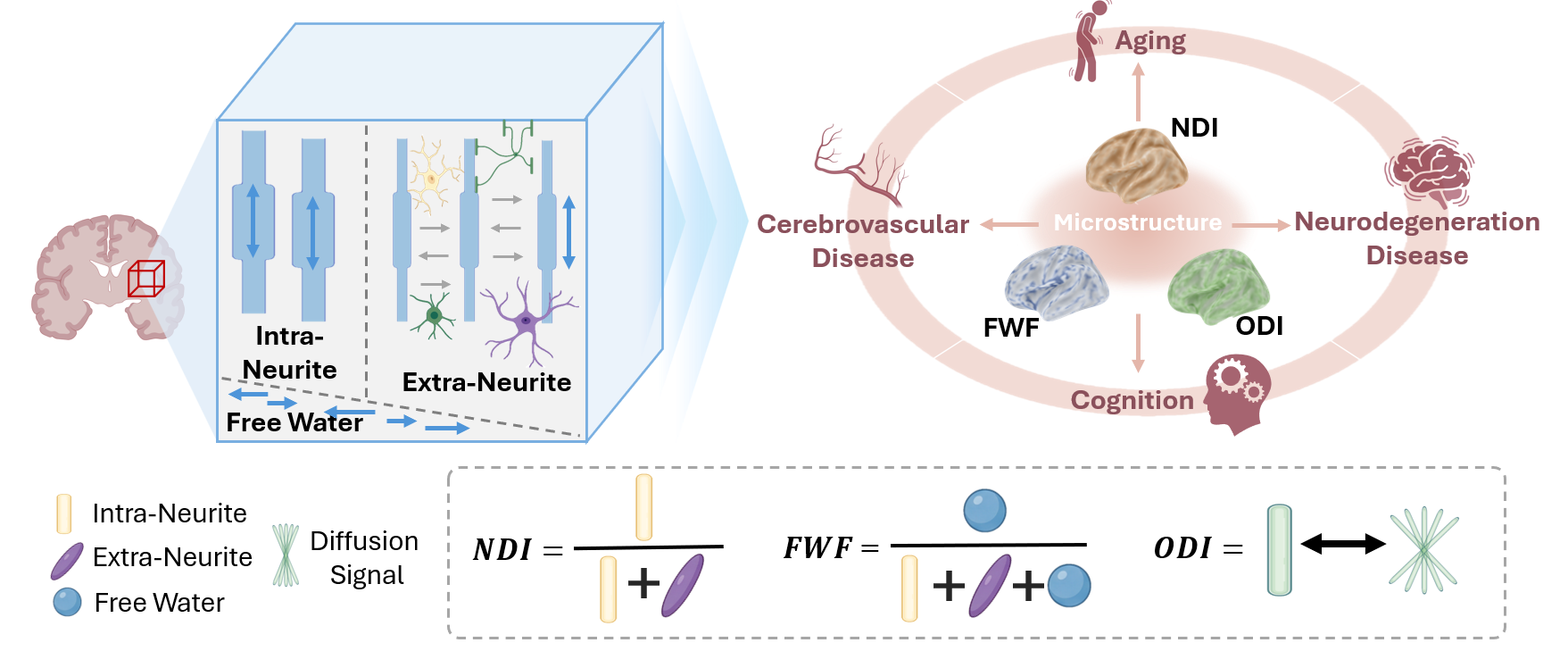}
    \caption{\textbf{Overview of NODDI-derived microstructural maps} (the input to BrainFIBRE). NODDI is a biophysical model that partitions tissue microstructure into three compartments (intra-neurite, extra-neurite, and free water) to derive three corresponding parametric maps: Neurite Density Index (NDI), Orientation Dispersion Index (ODI), and Free Water Fraction (FWF). Compared to conventional diffusion maps, these biologically interpretable maps offer greater sensitivity to microstructural changes, providing critical insights into brain aging, cerebrovascular pathology, neurodegenerative processes, and interindividual variation in cognitive performance.}
    \label{fig:fig1}
\end{figure}

The human brain is built upon a highly complex microstructural architecture. Unraveling this structure is essential to understanding the biological mechanisms behind cognition, healthy aging, and neurological disease. Diffusion Magnetic Resonance Imaging (dMRI) has emerged as a cornerstone technology for this purpose, offering a unique, non-invasive window into tissue microstructure by measuring the restricted random motion of water molecules within biological tissues \cite{le2001diffusion}. Despite its widespread clinical and research utility, conventional measures derived from diffusion fundamentally lack biological specificity. They summarize water diffusion across all encoding directions within a voxel, rather than isolating distinct microstructural compartments. For example, mean diffusivity (MD) quantifies the average rate of molecular diffusion, and fractional anisotropy (FA) reflects the degree of diffusion anisotropy, yet both are sensitive to multiple underlying tissue features without specifying their individual contributions. Consequently, they conflate distinct biophysical phenomena \cite{jones2013white}.

Neurite Orientation Dispersion and Density Imaging (NODDI) provides a more biologically specific characterization of tissue microstructure by modeling the voxel-wise diffusion signal into intra-neurite, extra-neurite, and free water compartments (Fig.\ref{fig:fig1}). This framework derives three parametric maps: neurite density index (NDI), orientation dispersion index (ODI), and free water fraction (FWF), which respectively characterize neurite packing, fiber coherence, and extracellular fluid. This microstructural modelling approach disentangles processes that are typically conflated in conventional diffusion metrics, such as neurite loss, demyelination, edema, and neuroinflammation, thereby yielding more directly interpretable tissue parameters\cite{zhang2012noddi}.

Despite the biological specificity in NODDI, harnessing its rich, multiparametric maps for predictive modeling remains a profound computational challenge: how to effectively fuse these complex, three-dimensional spatial maps without re-entangling their distinct biophysical meanings. Addressing this challenge is crucial, as NODDI-derived maps capture clinically consequential signals that are inaccessible to macroscopic contrasts: they reveal systematic yet regionally heterogeneous microstructural trajectories across aging in all three compartments \cite{merluzzi2016age,greenman2025aging,yu2024noddi,peters2014age}, provide mechanistically grounded sensitivity to cerebrovascular injury (\emph{e.g.}, white matter hyperintensities (WMH)) and cognitive decline \cite{ji2017distinct,duering2018free}, and detect early Alzheimer's-related microstructural alterations \cite{merluzzi2016age,yu2024noddi}. 

Recent brain foundation models have revolutionized neuroimaging analysis by leveraging self-supervised learning on large-scale, unlabeled datasets to extract transferable representations \cite{carobrainlm,dong2024brain,dongbrain,rui2025multi}. However, they mainly focus on either macroscopic structural MRI (\emph{e.g.}, T1- and T2-weighted scans) or functional MRI, leaving a model dedicated to tissue microstructure conspicuously absent. Furthermore, effectively integrating the NODDI maps into a unified learning framework is itself non-trivial. Each map encodes a distinct biophysical process, yet previous multi-modal fusion strategies in brain imaging \cite{rui2025multi,dongbrain,khajehnejad2025brainsymphony} collapse different modalities into a monolithic latent space, risking the re-entanglement of the very processes that the NODDI compartmental model was explicitly designed to decouple. A principled framework is therefore needed to systematically characterize the interactions among neurite density, orientation dispersion, and free water fraction. Such a framework must disentangle whether these maps provide complementary, redundant, or synergistic information for different neurological tasks, while preserving the biophysical specificity afforded by NODDI.

Partial Information Decomposition (PID) offers a theoretical foundation for understanding multimodal interactions \cite{williams2010nonnegative,bertschinger2014quantifying}. It decomposes the total information that multiple sources carry about a target into unique, redundant, and synergistic components. Uniqueness isolates what each modality contributes alone, redundancy captures information shared across modalities, and synergy reveals information that emerges only from their joint observation. Recently, a Mixture-of-Experts (MoE) architecture operationalizing PID has been proposed, assigning dedicated experts to each interaction type and extending the framework beyond the two-modality setting \cite{xin2025i2moe}. It achieves interpretable predictions on Alzheimer's Disease classification with four input modalities including imaging, genetic, clinical, and biospecimen data. However, it approximates PID interaction types through binary label agreement among unimodal classifiers. This coupling strictly ties the decomposition to classifier quality and, crucially, precludes the large-scale, self-supervised pretraining essential for foundation models.

\begin{figure}[t]
    \centering
    \includegraphics[width=1\linewidth]{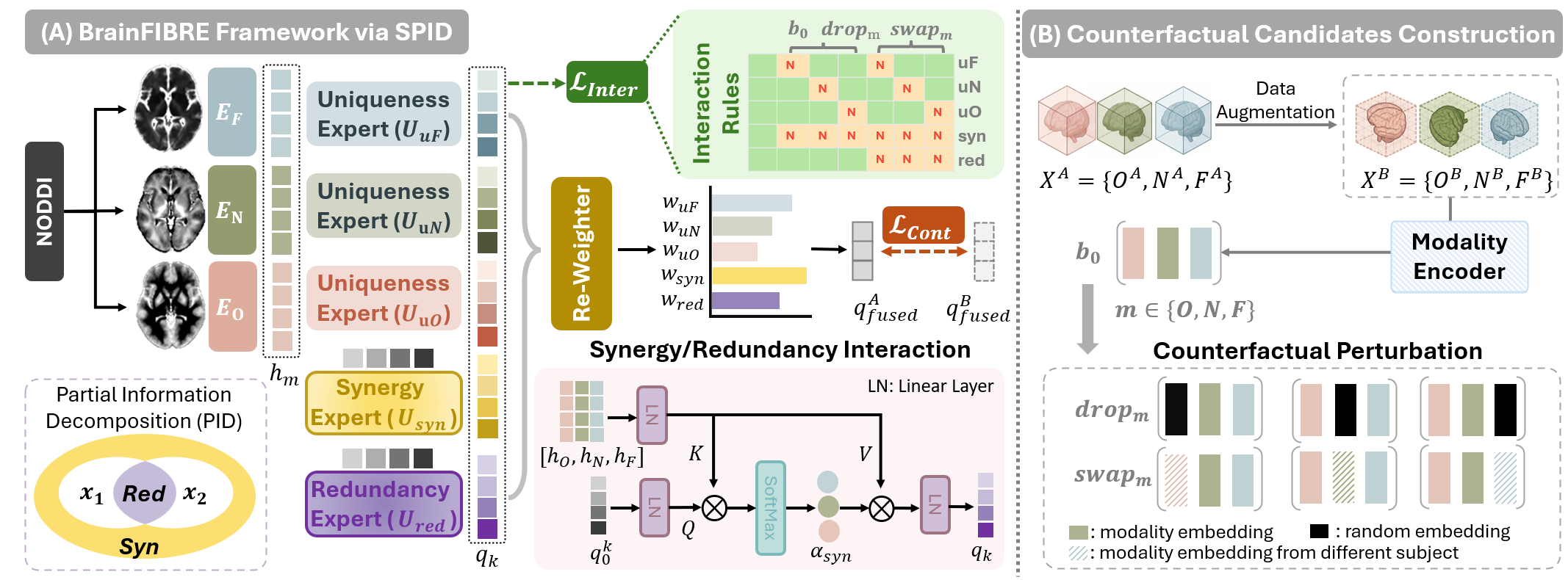}
    \caption{\textbf{Overview of BrainFIBRE}. (\textbf{A}) By the design of Self-supervised Partial Information Decomposition (SPID), BrainFIBRE extends classical PID to isolate unique, synergistic, and redundant information for self-supervised  pretraining. Three NODDI maps are processed by unimodal encoders into embeddings \(h_m\), which are routed to five interaction experts: three for uniqueness and two for synergy and redundancy. A Re-Weighter adaptively aggregates expert embeddings \(q_k\) via learned weights \(w_k\) into fused representation \(q_{fused}\). Pretraining is optimized by an interaction loss  \(\mathcal{L}_{Inter}\) guided by interaction rules, and a global contrastive loss \(\mathcal{L}_{Cont}\). (\textbf{B}) Counterfactual Candidate Construction (CCC) generates perturbed triplets (base \(b_0\), modality dropout (\(drop_m\)), and modality swap (\(swap_m\))) to provide self-supervised signals for SPID.}
    \label{fig:fig2}
\end{figure}

In this work, we introduce \textbf{BrainFIBRE} (\textbf{Brain} \textbf{F}oundation Model via \textbf{I}nformation Decomposition for \textbf{BR}ain Microstructur\textbf{E}), the first foundation model dedicated to brain tissue microstructure (Fig.\ref{fig:fig2}). It was pretrained on NODDI-derived microstructure maps from 55,592 participants in the UK Biobank, the largest publicly available diffusion MRI dataset to date. BrainFIBRE treats the three NODDI maps — NDI, ODI, and FWF — as distinct modalities and learns unified multimodal representations that capture their joint predictive power while preserving the biophysical specificity of each individual compartment. To achieve this, we propose \emph{Self-supervised Partial Information Decomposition (SPID)}, a novel pretraining objective that extends PID-guided multimodal learning to the self-supervised setting. SPID drives an MoE architecture comprising three Uniqueness experts, one Redundancy expert, and one Synergy expert, each trained to isolate its targeted information component without requiring downstream labels. Central to SPID is a \emph{Counterfactual Candidate Construction (CCC)} paradigm that systematically generates counterfactual views through modality dropping and swapping, providing the contrastive signal needed to disentangle information in a purely self-supervised manner.

We evaluate BrainFIBRE on a comprehensive suite of downstream tasks spanning prediction of age, sex, cerebrovascular disease \& neurodegenerative markers, and cognition across both Caucasian and Asian cohorts. It achieves state-of-the-art performance across tasks while yielding neurobiological interpretations that align with known microstructural signatures of aging and disease.

Our contributions can be summarized as follows:
\begin{itemize}
\item We present BrainFIBRE, the first brain foundation model for tissue microstructure, pretrained on NODDI-derived maps from UK Biobank.
\item We propose Self-supervised Partial Information Decomposition (SPID) with a Counterfactual Candidate Construction (CCC) paradigm, enabling PID-guided self-supervised pretraining of MoE without any downstream label.
\item We demonstrate state-of-the-art performance across diverse downstream tasks on both Caucasian and Asian cohorts, with neurobiological interpretations.
\end{itemize}

\section{Related Work}

\subsection{Brain Foundation Model}

Recent advances in brain foundation models have demonstrated the power of self-supervised learning to extract transferable representations from large-scale neuroimaging. These efforts have spanned diverse modalities. In the domain of electroencephalography (EEG), among the most representative works, LaBraM \cite{jianglarge} pretrains a large brain model via masked EEG modeling, while NeuroLM \cite{jiangneurolm} further integrates EEG signals into a large language model to enable unified multi-task inference through instruction tuning. In functional MRI (fMRI), BrainLM \cite{carobrainlm} and Brain-JEPA \cite{dong2024brain} pioneer representation learning for brain dynamics through masked prediction and joint-embedding architectures, respectively. In structural MRI (sMRI), BrainMVP \cite{rui2025multi} proposes self-supervised pretraining on 3D brain volumes, learning cross-parametric correspondences among multi-modal MRI. More recently, multimodal brain foundation models that jointly fuse structural and functional neuroimaging have also been proposed \cite{dongbrain,khajehnejad2025brainsymphony}. Despite this progress, existing brain foundation models operate exclusively on macroscopic contrasts, and none address the tissue microstructures. A related line of diffusion MRI work studies learning-based microstructural parameter estimation, such as codebook-driven estimation of tissue characterization indices from diffusion measurements \cite{wang2026deep}. BrainFIBRE is complementary: rather than deriving microstructural maps from raw DWI, it learns transferable representations from already-derived NODDI maps and models the interactions among ODI, NDI, and FWF. Thus, interaction-aware representation learning over NODDI-derived tissue microstructure remains largely unexplored.

\subsection{Multimodal Interaction MoEs}

Existing multimodal MoEs typically rely on heuristic routing \cite{he2021fastmoe,shazeer2017outrageously}, offering limited insight into how different modalities interact with each other \cite{liu2024deepseek,shen2024mome}. Modeling multimodal interactions provides a principled and natural inductive bias for guiding expert specialization in MoEs. In MMoE \cite{yu2024mmoe}, dedicated experts are assigned to capture predetermined types of multimodal interactions. Building on this idea, I$^2$MoE introduces interaction-type categorization into the MoE training process and generalized the framework to support an arbitrary number of modalities \cite{xin2025i2moe}. However, its approximation of PID interaction types relies on label agreement among unimodal classifiers, coupling the decomposition to classifier quality and, crucially, ruling out the large-scale self-supervised pretraining on which foundation models depend. General multimodal self-supervised frameworks such as M3-JEPA \cite{lei2024m3} use MoE-style latent alignment across heterogeneous modalities, but do not assign experts to explicit PID interaction atoms. In contrast, BrainFIBRE uses counterfactual modality dropping and swapping to specialize experts toward unique, redundant, and synergistic information in a label-free manner.

\section{BrainFIBRE}
We propose BrainFIBRE, a brain microstructure foundation model driven by the proposed Self-supervised Partial Information Decomposition (SPID)  (Fig.~\ref{fig:fig2}.). For each participant, the input is a 3D microstructural triplet, \(\mathbf{X} = \left\{ X_m\right\}_m\), where $m\in$ \{O, N, F\} and each $X_m \in \mathbb{R}^{H \times W \times D}$ corresponds to a specific NODDI-derived map (O:ODI, N:NDI, F:FWF). To disentangle the profound relationships among modalities, we design a Mixture-of-Experts (MoE) architecture to explicitly model distinct types of modality interactions (Sec.~\ref{sec3.1}) while extending PID to the self-supervised setting (Sec.~\ref{sec3.2}). The pretraining objective is anchored by a Counterfactual Candidate Construction (CCC) strategy, which generates perturbations to provide optimization signals.

\subsection{Architecture}
\label{sec3.1}

\subsubsection{Disentangled Encoding of Modality Interaction} As shown in Fig.~\ref{fig:fig2}A, BrainFIBRE consists of modality encoders and interaction experts. First, three 3D Vision Transformers (ViT)~\cite{dosovitskiy2020image} encoders, denoted as $E_O$, $E_N$, and $E_F$, are applied on each modality map in the input triplet \(\mathbf{X}\) to extract unimodal embeddings $h_m \in \mathbb{R}^{1 \times d_{enc}}$.

Guided by PID, we categorize interactions into three types: \textit{uniqueness}, \textit{synergy} , and \textit{redundancy}, and define five distinct experts to extract specific interaction embeddings \(q_k\): three uniqueness experts ($U_{uO}, U_{uN}, U_{uF}$) for modality-exclusive features, one synergy expert ($U_{syn}$) for emergent multimodal features, and one redundancy expert ($U_{red}$) for shared information. The entire disentangled encoding process is formulated as:
\begin{equation}
h_m = E_m(X_m) \in \mathbb{R}^{1 \times d_{enc}}, \quad q_k = U_k(\text{InterRule}([{h_O, h_N, h_F}]) \in \mathbb{R}^{1 \times d_{exp}}
\end{equation}
where $k \in \mathcal{K}= \{\text{uO}, \text{uN}, \text{uF}, \text{syn}, \text{red}\}$, \(d_{enc}\) and \(d_{exp}\)  are the dimension of unimodal and expert embeddings, and InterRule$(\cdot)$ denotes expert-specific perturbations based on interaction rules (detailed in Sec.~\ref{sec3.2}).

Specifically, uniqueness experts capture the information unique to each modality via linear projection heads. On the other hand, the synergy and redundancy experts model more complex interactions via query-based information encoding (Fig.\ref{fig:fig2}.A). For either of these two, a learnable query token $q^{k}_0 \in \mathbb{R}^{1 \times d_{enc}}$ is initialized as the Query ($Q$). The three unimodal embeddings are concatenated and linearly projected to form the Key ($K \in \mathbb{R}^{3 \times d_{enc}}$) and Value ($V \in \mathbb{R}^{3 \times d_{enc}}$). The expert embedding $q_k$ is computed via scaled dot-product attention:
\begin{equation}
\alpha_k = \text{Softmax}\left(\frac{q^{k}_0 K^\top}{\sqrt{d_{enc}}}\right)\in \mathbb{R}^{1 \times 3}, \quad q_k = \text{Linear}(\alpha_k V) \in \mathbb{R}^{1 \times d_{exp}}
\end{equation}
where $\alpha_k$ is the attention weights and \(\text{Linear}(\cdot)\) is a linear projection head. This design allows the synergy and redundancy experts to dynamically aggregate information across all modalities driven by their respective learned queries.

\subsubsection{Adaptive Fusion of Expert Embeddings} To integrate the expert embeddings, we introduce a Re-weighter module to dynamically assign importance to each expert. Implemented as a Multi-Layer Perceptron (MLP), the Re-weighter takes the concatenated unimodal embeddings to generate a probability distribution over the five interaction experts. This yields a set of expert weights \(w_k\), which are used to produce the final fused representation \(q_{fused}\):
\begin{equation}
\mathbf{w} = \text{Softmax}(\text{MLP}([h_O, h_N, h_F]))\in \mathbb{R}^{1 \times 5}, \quad q_{fused} = \sum_{k \in \mathcal{K}} w_k \cdot q_k
\end{equation}

\subsection{Self-supervised Partial Information Decomposition (SPID)}
\label{sec3.2}
\subsubsection{Counterfactual Candidate Construction (CCC)} While PID-informed MoE frameworks enable interpretable multimodal learning~\cite{xin2025i2moe},  they are typically trained in a supervised manner, limiting their generalizability for large-scale pre-training. To address this, we propose Self-supervised Partial Information Decomposition (SPID). The core challenge of extending PID to a self-supervised setting lies in isolating unique, redundant, and synergistic interactions without external supervision. Standard spatial augmentations that preserve the aligned modality triplet fail to provide the necessary teaching signal to disentangle these interactions, as the model is never forced to distinguish between modality-specific and the natural co-occurrence among modalities.

To address this, we introduce Counterfactual Candidate Construction (CCC) (Fig.~\ref{fig:fig2}B) to systematically disrupt biophysical alignment and generate the contrastive signals, enabling the model to isolate and learn each specific PID interaction in an entirely self-supervised manner. For a minibatch of $N$ participants, let $\mathbf{X}_i$ denote the input for the $i$-th participant. We generate two augmented views, $\mathbf{X}^A(i)$ and $\mathbf{X}^B(i)$, each comprising a triplet of NODDI maps. These views are processed by unimodal encoders to yield their respective embeddings:
\begin{equation}
\mathbf{H}^v(i) = \{h_m^v(i)\}_m = \{ E_m(X_m^v(i)) \}_m, \quad m \in \{O, N, F\} \quad \text{and} \quad v \in \{A, B\}
\end{equation}
We use $\mathbf{H}^A$ to obtain the anchor embedding via the five interaction experts. On the other hand, $\mathbf{H}^B$ is used to construct a set of counterfactural perturbations:
\begin{enumerate}
\item Base Candidate ($b_0$): The complete and unperturbed embedding derived from $\mathbf{X}^B(i)$:
\begin{equation}
    b_0(i) =  U_k([h_O^B(i), h_N^B(i), h_F^B(i)])
    \end{equation}
\item Modality Dropout ($drop_m$): Simulates the absence of a specific modality. $drop_m(i)$ replaces the $m$-th component of  \(\mathbf{H}^v\) with Gaussian noise. For instance:
\begin{equation}
    drop_O(i) = U_k([\text{noise}, h_N^B(i), h_F^B(i)])
    \end{equation}
\item Modality Swap ($swap_m$): Constructs a counterfactual sample to destroy biophysical correspondances. $swap_m(i)$ replaces the $m$-th component of \(\mathbf{H}^v\) with the same unimodal embedding from another randomly selected participant $j$ in the same batch. For instance:
\begin{equation}
    swap_O(i) =  U_k([h_O^B(j), h_N^B(i), h_F^B(i)])
    \end{equation}
\end{enumerate}
Together, these perturbations form a candidate pool $\mathcal{C}_i = \{b_0(i)\} \cup \{drop_m(i)\}_{m}$ $\cup \{swap_m(i)\}_{m}$ for each participant. \textbf{Crucially, identical spatial transformations are applied across the entire batch for each view, guaranteeing that swapped embeddings maintain strict anatomical alignment to prevent trivial solutions.} Since human brains share highly conserved macroscopic anatomy, swapping a modality with one from another participant yields a counterfactual that appears \textbf{\textcolor{blue}{\emph{anatomically valid}}} but is fundamentally \textbf{\textcolor{blue}{\emph{microstructurally mismatched}}}. By forcing the model to distinguish these subtle disruptions, SPID moves beyond recognizing global brain shapes to genuinely learning voxel-level biophysical correspondences among neurite density, orientation dispersion, and free water.

\subsubsection{Expert-Conditional Interaction Rules} Based on the candidate pool  $\mathcal{C}_i$ generated via CCC, we define specific interaction rules to construct positive samples or negative samples sets for each expert $k$. For instance, the uniqueness expert $U_O$ should be invariant to changes in other modalities ($N, F$) but sensitive to its own ($O$); thus, $drop_N$,$ drop_F$, $swap_N$, and $swap_F$ are considered as positives, whereas $drop_O$ and $swap_O$ are negatives. The redundancy expert $U_{red}$ should remain robust to missing partial information, making all dropout candidates positive and all swap candidates negative. For the synergy expert $U_{syn}$, which relies on constructive multimodal interactions, any perturbation destroys the emergent information, leaving only the base candidate $b_0$ as positive. Complete definitions for all experts are detailed in Table~\ref{tab:interaction_rules}.

Based on these rules, we formulate the interaction objective \(\mathcal{L}_{inter}\) using a multi-positive symmetric InfoNCE loss \cite{oord2018representation}. Let $\mathcal{C}_{batch}$ denotes the perturbation candidates from all the participants in the current batch. For each expert \(U_k\) and participant \(i\), the positive candidate pool \(\mathcal{P}_k(i) \subset \mathcal{C}_{batch}\), while the remainder as negatives. The objective is to maximize the similarity between the anchor \(q_k^A(i) = U_k([h_O^A(i),h_N^A(i),h_F^A(i)]))\) and its positive candidates in \(\mathcal{P}_k(i)\), while pushing away negative candidates and all other participants in the batch:  
\begin{equation}
\mathcal{L}_{Inter} = \sum_{k \in \mathcal{K}} \sum_{i=1}^{N} - \log \frac{\sum_{z^+ \in \mathcal{P}_k(i)} \exp(\operatorname{sim}(q_k^A(i), z^+) / \tau)}{\sum_{z \in \mathcal{C}_{batch}} \exp(\operatorname{sim}(q_k^A(i), z) / \tau)}
\end{equation}
where $\tau$ is the temperature parameter. Although formulated with View A as the anchor for brevity, \(\mathcal{L}_{Inter}\) is symmetrized by averaging the losses computed in both directions (A$\to$B and B$\to$A). By leveraging PID principles to specialize each expert via the interaction rules, this objective enables the end-to-end training of BrainFIBRE in a purely self-supervised, label-free manner.

\subsubsection{Pretraining Objectives} 
To learn robust participant-specific representations, we apply the symmetric InfoNCE loss~\cite{chen2020simple}, denoted as \(\mathcal{L}_{\text{Cont}}\), to the fused embeddings \(q_{\text{fused}}^{A}\) and \(q_{\text{fused}}^{B}\) produced by the two augmented views. This objective maximizes the similarity between the anchor \(q_{\text{fused}}^{A}\) and its positive counterpart \(q_{\text{fused}}^{B}\), while pushing away embeddings from other participants in the batch. Like \(\mathcal{L}_{Inter}\), \(\mathcal{L}_{Cont}\) is symmetrized by averaging across both directions.
 
Training MoE frameworks in a self-supervised manner is prone to representation collapse, where the model trivially assign all samples to a single expert. To mitigate this, we introduce two auxiliary regularization terms. The Balance Loss \(\mathcal{L}_{balance}\) minimizes the Kullback-Leibler (KL) divergence between the batch-averaged expert assignment probabilities and a uniform distribution, preventing global expert collapse. The Entropy Loss \(\mathcal{L}_{entropy}\) minimizes the negative entropy of individual weight distributions, encouraging soft probabilistic assignments such that each participant integrates information from multiple experts rather than relying on a deterministic, one-hot routing.

The final training objective is a weighted sum of the interaction loss, global contrastive loss, and the expert regularizations:  
\begin{equation}
\mathcal{L}_{total} = \mathcal{L}_{Inter} + \mathcal{L}_{Cont} + \beta \mathcal{L}_{balance} + \gamma \mathcal{L}_{entropy}
\end{equation}
where $\beta$ and $\gamma$ are hyperparameters for expert regularization.

\textbf{Theoretical grounding} The CCC-based interaction rules provide a theoretical basis for using SPID as a self-supervised surrogate of PID. Specifically, \emph{Theorem 1} in the supplementary material establishes the correctness of the expert-conditional rules: for each expert, the positive candidates are exactly those perturbations that preserve its designated latent interaction factor, while the negative candidates destroy or replace that factor almost surely. Building on this rule-level correctness, \emph{Theorem 2} shows that the population version of the interaction contrastive task admits a PID-factorized optimum, where each expert representation can solve its contrastive task using only its corresponding factor: the three modality-unique factors, the redundant factor, or the synergistic factor. Consequently, \emph{Corollary 1} implies that, if this factorized solution is reached, the five expert representations instantiate a coarse five-atom PID decomposition of the latent participant-specific state. This result should be interpreted as a population-level grounding of our design rather than a claim of exact recovery of the full three-source PID lattice or finite-sample identifiability. Detailed assumptions and proofs are provided in Supplementary~\ref{sec:spid_pid_grounding}.


\begin{table}[t]
    \centering
    \caption{Expert-Conditional Interaction Rules.}
    \label{tab:interaction_rules}
    
    \scriptsize 
    \renewcommand{\arraystretch}{1.4}
    \setlength{\tabcolsep}{2pt}
    
    \setlength{\aboverulesep}{0pt}
    \setlength{\belowrulesep}{0pt}
    
    \begin{tabular}{l c l l}
        \toprule
        
        \rowcolor{gray!25}
        \multicolumn{1}{c}{\textbf{Expert}} & 
        \multicolumn{1}{c}{\textbf{Anchor}} & 
        \multicolumn{1}{c}{\textbf{Positive}} & 
        \multicolumn{1}{c}{\textbf{Negative}} \\
        \midrule
        
        $U_{uO}$ 
        & $q_O^A$ 
        & $\{b_0, drop_N, drop_F, swap_N, swap_F\}$ 
        & $\{drop_O, swap_O\}$ \\
        
        $U_{uN}$ 
        & $q_N^A$ 
        & $\{b_0, drop_O, drop_F, swap_O, swap_F\}$ 
        & $\{drop_N, swap_N\}$ \\
        
        $U_{uF}$ 
        & $q_F^A$ 
        & $\{b_0, drop_O, drop_N, swap_O, swap_N\}$ 
        & $\{drop_F, swap_F\}$ \\
        
        $U_{\text{syn}}$ 
        & $q_{\text{syn}}^A$ 
        & $\{b_0\}$ 
        & $\{drop_m, swap_m\}_{m}$ \\ 
        
        $U_{\text{red}}$ 
        & $q_{\text{red}}^A$ 
        & $\{b_0, drop_O, drop_N, drop_F\}$ 
        & $\{swap_m\}_{m}$ \\
        \bottomrule
    \end{tabular}
\end{table}

\section{Experiments}

\subsection{Datasets}

\subsubsection{Pretraining Dataset} The UK Biobank (UKB) is a population-based prospective study of over 500,000 participants (across the United Kingdom) with extensive phenotypic, clinical, and imaging data \cite{sudlow2015uk}. A subset (middle-aged to older adults) underwent brain MRI as part of an ongoing imaging initiative, providing substantial variability in aging and cardiometabolic risk factors \cite{miller2016multimodal}. Diffusion MRI was acquired via a harmonized multi-shell protocol on identical 3T Siemens Skyra scanners, from which NODDI metrics (ODI, NDI, and FWF) were systematically estimated. BrainFIBRE was pretrained on these NODDI-derived maps from 55,592 UKB participants.

\subsubsection{Downstream Datasets} To rigorously evaluate the representational power and clinical generalizability of BrainFIBRE, we benchmarked the model across diverse downstream datasets and tasks. We first evaluated the model on a held-out UKB test set (4,307 participants) across demographic (age, sex), neurodegenerative (hippocampal atrophy at 5-year follow-up), and cognitive (processing speed) predictions. Next, we assessed the model on the independent Human Connectome Project in Aging (HCP-Aging) dataset \cite{bookheimer2019lifespan}(630 participants) to evaluate demographic and executive function (Flanker, CardSort) prediction. Crucially, to demonstrate the utility of BrainFIBRE beyond predominantly Caucasian populations and into targeted clinical applications, we evaluated it on SINGER \cite{xu2022singapore}. This Asian dataset (818 participants) comprises community-dwelling older adults at risk for vascular cognitive impairment, ranging from normal cognition to mild cognitive impairment. We utilized this dataset to predict brain age, cognition, and white matter hyperintensity (WMH) volume, a key cerebrovascular biomarker. Results were averaged over three independent runs, each using a distinct random seed for data split (train/val/test ratio of 6:2:2). Details of imaging preprocessing are provided in the supplementary materials. 

\subsection{Implementation Details}
In BrainFIBRE, we employed 3D ViT-S \cite{dosovitskiy2020image} as the backbone for all unimodal encoders, with an input volume size of (96, 112, 96) and a patch size of (16, 16, 16). \(d_{enc}\) and \(d_{exp}\) is 384 and 128, respectively. The pretraining was conducted on 8 NVIDIA H200 (140GB) GPUs. The model was pretrained for 200 epochs with a global batch size of 280, which required 40 hours. We utilized the AdamW optimizer \cite{loshchilov2017decoupled} ($\beta_1=0.9, \beta_2=0.999$) with a base learning rate of $2e^{-4}$, a weight decay of $1e^{-4}$, and a cosine annealing learning rate scheduler. 

Regarding the pretraining objectives, the temperature parameters $\tau$ for $\mathcal{L}_{Inter}$ and $\mathcal{L}_{Cont}$ were empirically set to 0.2 and 0.12, respectively. The expert regularization weights were set to $\beta=0.05$ and $\gamma=0.01$. To prevent early expert collapse, we adopted a warmup strategy during the first 5 epochs, where uniform routing weights were explicitly assigned to all five interaction experts. 

For regression tasks, we adopted Mean Absolute Error (MAE) and Pearson Correlation (Corr) as the evaluation metrics. For classification tasks, we adopted the F1 score and Accuracy (Acc) metrics. More details regarding training configurations are provided in the supplementary materials.

\begin{table}[t]
    \centering
    \begin{threeparttable}
        \caption{Performance comparisons on UKB held-out dataset. Best results are bolded; underlined values denote statistical significance against all others ($p < 0.05$).}
        \label{tab:comparison_ukb}
        
        \fontsize{6.5pt}{7.5pt}\selectfont 
        \renewcommand{\arraystretch}{1.2}
        \setlength{\tabcolsep}{0.2pt} 
        
        \begin{tabular}{l c c c c c c c c c}
            \toprule
            \multirow{2}{*}{\textbf{Model}} & 
            \multirow{2}{*}{\textbf{Pretrain}} &
            \multicolumn{2}{c}{\textbf{Age}} & 
            \multicolumn{2}{c}{\textbf{Sex}} & 
            \multicolumn{2}{c}{\textbf{Hipp.Atrophy}} &
            \multicolumn{2}{c}{\textbf{Proc.Speed}} \\
            
            \cmidrule(lr){3-4} \cmidrule(lr){5-6} \cmidrule(lr){7-8} \cmidrule(lr){9-10}
            
            & & MAE$\downarrow$ & Corr$\uparrow$ & F1$\uparrow$ & Acc$\uparrow$ & MAE $\downarrow$ & Corr$\uparrow$ & MAE$\downarrow$ & Corr$\uparrow$ \\
            \midrule
            
            ViT-O & \xmark & 4.50\std{0.07} & 0.67\std{0.00} & 78.08\std{1.30} & 78.11\std{1.31}  & 110.29\std{3.06} & 0.16\std{0.01} & 0.75\std{0.01} & 0.20\std{0.06} \\
            
            ViT-N & \xmark & 4.49\std{0.09} & 0.67\std{0.01} & 82.01\std{1.29} & 82.13\std{1.31} & 110.01\std{2.57} & 0.17\std{0.03} & 0.75\std{0.01} & 0.20\std{0.04} \\
            
            ViT-F & \xmark & 4.24\std{0.01} & 0.71\std{0.01} & 82.18\std{0.59} & 82.21\std{0.60} & 109.85\std{3.43} & 0.19\std{0.01} & 0.71\std{0.01} & 0.32\std{0.02} \\
            
            I\textsuperscript{2}MOE \cite{xin2025i2moe} & \xmark & 6.21\std{0.15} & 0.42\std{0.05} & 71.94\std{1.32} & 72.58\std{1.39} & 111.90\std{2.23} & 0.17\std{0.02} & 0.73\std{0.01} & 0.33\std{0.02} \\
            
            BrainMVP \cite{rui2025multi} & \cmark & 4.46\std{0.27} & 0.68\std{0.04} & 78.93\std{2.57} & 78.65\std{2.33} & 109.76\std{2.96} & 0.19\std{0.04} & 0.70\std{0.01}& \textbf{0.36}\std{0.01} \\
            \midrule 
            
            TFS & \xmark & 4.33\std{0.02} & 0.69\std{0.01} & 79.55\std{2.02} & 79.62\std{1.99} & 111.05\std{2.08} & 0.18\std{0.01} & 0.71\std{0.01} & 0.32\std{0.02} \\
        
            \textbf{BrainFIBRE} & \cmark & \underline{\textbf{3.95}}\std{0.02} & \underline{\textbf{0.76}}\std{0.00} & \underline{\textbf{89.30}}\std{0.82} & \underline{\textbf{89.33}}\std{0.84} & \textbf{109.21}\std{2.03} & \textbf{0.21}\std{0.01} & \textbf{0.70}\std{0.01} & 0.33\std{0.02} \\
            
            \bottomrule
        \end{tabular}

        \begin{tablenotes}
            \tiny
            \item[] \hspace{-1.6em} TFS: trained-from-scratch. Hipp.Atropy: future hippocampal atrophy. Proc.Speed: processing speed.
        \end{tablenotes}
        
    \end{threeparttable}
\end{table}

\subsection{Downstream Evaluation}

\subsubsection{Performance Comparison with Baselines on Internal Dataset} BrainFIBRE demonstrates strong representation capability and generalizability across a diverse set of downstream tasks, encompassing demographic prediction, the prediction of neurodegenerative marker (hippocampus atrophy), and cognitive performance (processing speed). Table~\ref{tab:comparison_ukb} summarizes the quantitative results on the held-out test set from UKB. To assess our approach, we benchmarked BrainFIBRE against a comprehensive set of baselines: (1) three unimodal 3D ViTs (ViT-O, ViT-N, and ViT-F) trained exclusively on a single NODDI map; (2) I\textsuperscript{2}MOE, a fully supervised multimodal model; (3) BrainMVP, a state-of-the-art multimodal self-supervised foundation model; and (4) a standard train-from-scratch paradigm using the BrainFIBRE architecture. Significance was evaluated via paired t-test, annotated only when BrainFIBRE significantly outperformed all baselines $(p < 0.05)$. Overall, BrainFIBRE outperforms the unimodal, supervised \& self-supervised multimodal baselines. The substantial gains underscore BrainFIBRE's capacity to effectively extract and integrate subtle microstructural features from diffusion imaging data.

\begin{table}[t]
    \centering
    \begin{threeparttable}
        \caption{Performance comparisons on HCP-A dataset. Best results are bolded; underlined values denote statistical significance against all others ($p < 0.05$).}
        \label{tab:comparison_hcp}
        
        \fontsize{6.5pt}{7.5pt}\selectfont 
        \renewcommand{\arraystretch}{1.2}
        \setlength{\tabcolsep}{0.2pt}
        
        \begin{tabular}{l c c c c c c c c c}
            \toprule
            \multirow{2}{*}{\textbf{Model}} & 
            \multirow{2}{*}{\textbf{Pretrain}} & 
            \multicolumn{2}{c}{\textbf{Age}} & 
            \multicolumn{2}{c}{\textbf{Sex}} & 
            \multicolumn{2}{c}{\textbf{Flanker}} & 
            \multicolumn{2}{c}{\textbf{CardSort}} \\
            
            \cmidrule(lr){3-4} \cmidrule(lr){5-6} \cmidrule(lr){7-8} \cmidrule(lr){9-10}
            
            & & MAE$\downarrow$ & Corr$\uparrow$ & F1$\uparrow$ & Acc$\uparrow$ & MAE$\downarrow$ & Corr$\uparrow$ & MAE$\downarrow$ & Corr$\uparrow$ \\
            \midrule
            
            ViT-O & \xmark & 11.92\std{0.51} & 0.49\std{0.09} & 61.33\std{6.70} & 47.09\std{15.40} & 6.14\std{0.34} & 0.23\std{0.07} & 6.10\std{0.41} & 0.18\std{0.05} \\
            
            ViT-N & \xmark & 11.09\std{0.69} & 0.64\std{0.03} & 57.33\std{2.10} & 36.43\std{0.86} & 6.13\std{0.34} & 0.18\std{0.07} & 6.08\std{0.36} & 0.20\std{0.07}\\
            
            ViT-F & \xmark & 6.84\std{0.57} & 0.82\std{0.03} & 57.87\std{4.20} & 52.72\std{4.13} & 5.92\std{0.27} & 0.33\std{0.03} & 5.99\std{0.40} & 0.30\std{0.04}\\
            
            I\textsuperscript{2}MOE \cite{xin2025i2moe} & \xmark & 10.37\std{0.42} & 0.82\std{0.01} & 67.47\std{4.90} & 66.15\std{3.80} & 5.86\std{0.26}& 0.30\std{0.06} & 5.95\std{0.4} & 0.26\std{0.06} \\
            
            BrainMVP \cite{rui2025multi} & \cmark & 8.26\std{3.30} & 0.81\std{0.04} & 71.20\std{9.49} & 79.90\std{4.69} & 5.92\std{0.33} & 0.35\std{0.01} & 6.12\std{0.67} & 0.37\std{0.04}\\
            
            \midrule
            
            TFS & \xmark & 6.57\std{0.35} & 0.84\std{0.02} & 74.40\std{3.39} & 73.86\std{3.02} & 5.99\std{0.23} & 0.31\std{0.04} & 6.08\std{0.47} & 0.34\std{0.02}\\
            
            \textbf{BrainFIBRE} & \cmark & \underline{\textbf{5.54}}\std{0.50} & \underline{\textbf{0.87}}\std{0.02} & \underline{\textbf{83.73}}\std{2.29} & \underline{\textbf{83.36}}\std{2.09} & \textbf{5.75}\std{0.39} & \underline{\textbf{0.40}}\std{0.05} & \textbf{5.92}\std{0.29} & \underline{\textbf{0.38}}\std{0.02} \\
            
            \bottomrule
        \end{tabular}

         \begin{tablenotes}
            \tiny
            \item[] \hspace{-1.6em} TFS: trained-from-scratch.
        \end{tablenotes}
        
    \end{threeparttable}
\end{table}

\begin{table}[t]
    \centering
    \begin{threeparttable}
        \caption{Performance comparisons on SINGER. Best results are bolded; underlined values denote statistical significance against all others ($p < 0.05$).}
        \label{tab:comparison_asian}
        
        \fontsize{6.5pt}{7.5pt}\selectfont 
        \renewcommand{\arraystretch}{1.2}
        \setlength{\tabcolsep}{0.2pt}
        
        \begin{tabular}{l c c c c c c c c c}
            \toprule
            \multirow{2}{*}{\textbf{Model}} & 
            \multirow{2}{*}{\textbf{Pretrain}} & 
            \multicolumn{2}{c}{\textbf{Age}} & 
            \multicolumn{2}{c}{\textbf{Mean Thickness}} & 
            \multicolumn{2}{c}{\textbf{WMH}} &
            \multicolumn{2}{c}{\textbf{Proc.Speed}} \\
            
            \cmidrule(lr){3-4} \cmidrule(lr){5-6} \cmidrule(lr){7-8} \cmidrule(lr){9-10}
            
            & & MAE$\downarrow$ & Corr$\uparrow$ & MAE$\downarrow$ & Corr$\uparrow$ & MAE$\downarrow$ & Corr$\uparrow$ & MAE$\downarrow$ & Corr$\uparrow$ \\
            \midrule
            
            ViT-O & \xmark & 3.84\std{0.06} & 0.37\std{0.03} & 0.08\std{0.00} & 0.26\std{0.05}& 2.47\std{0.33} & 0.21\std{0.02} & 0.67\std{0.08} & 0.09\std{0.03} \\
            
            ViT-N & \xmark & 3.60\std{0.04} & 0.47\std{0.01} & 0.08\std{0.00} & 0.36\std{0.05}& 2.42\std{0.28} & 0.37\std{0.12} & 0.67\std{0.07} & 0.10\std{0.03}\\
            
            ViT-F & \xmark & 3.37\std{0.11} & 0.55\std{0.03} & 0.07\std{0.00} & 0.47\std{0.03}& 2.37\std{0.31} & 0.34\std{0.04} & 0.67\std{0.08} & 0.21\std{0.01} \\
            
            I\textsuperscript{2}MOE \cite{xin2025i2moe} & \xmark & 4.10\std{0.13} & 0.36\std{0.07} & 0.07\std{0.00} & 0.48\std{0.04}& 2.66\std{0.38} & 0.34\std{0.00} & 0.66\std{0.06}& 0.19\std{0.03}\\
            
            BrainMVP \cite{rui2025multi} & \cmark & 3.44\std{0.40} & 0.50\std{0.03} & 0.07\std{0.02} & 0.52\std{0.03} & 2.54\std{0.56} & 0.44\std{0.13} & 0.70\std{0.09} & 0.15\std{0.03} \\
            
            \midrule
            
            TFS & \xmark & 3.56\std{0.11} & 0.49\std{0.03} & 0.07\std{0.01} & 0.41\std{0.10} & 4.49\std{0.40} & 0.25\std{0.16} & 0.86\std{0.33} & 0.19\std{0.01} \\
            
            \textbf{BrainFIBRE} & \cmark & \underline{\textbf{3.21}}\std{0.12} & \underline{\textbf{0.60}}\std{0.04} & \textbf{0.07}\std{0.00} & \underline{\textbf{0.55}}\std{0.02}& \underline{\textbf{2.08}}\std{0.30} & \underline{\textbf{0.59}}\std{0.02} & \textbf{0.66}\std{0.07} & \underline{\textbf{0.25}}\std{0.03} \\
            
            \bottomrule
        \end{tabular}

        \begin{tablenotes}
            \tiny
            \item[] \hspace{-1.6em} TFS: trained-from-scratch. WMH: white matter hyperintensity volume. Proc.Speed: processing speed.
        \end{tablenotes}
        
    \end{threeparttable}
\end{table}

\subsubsection{Superior Performance and Generalizability on External Dataset across Cohorts} 
We further evaluated BrainFIBRE on the HCP-Aging dataset on demography prediction (age, sex) and executive function tasks (Flanker and CardSort scores). As shown in Table \ref{tab:comparison_hcp}, BrainFIBRE consistently achieves superior predictive performance compared to baseline methods, highlighting its capacity to capture both fundamental demographic features and complex cognitive outcomes. This suggests BrainFIBRE's strong generalizability to unseen cohorts and the sensitivity of its learned representations to clinically relevant traits.

To further evaluate cross-ethnic generalizability and disease-relevant predictive utility, we assessed BrainFIBRE's performance on SINGER across age, mean cortical thickness, white matter hyperintensity (WMH) volume, and processing speed (Proc.Speed) prediction. As shown in Table \ref{tab:comparison_asian}, BrainFIBRE outperforms all baseline models on all four tasks. This robust advantage in an ethnically distinct cohort underscores the model’s resilience to population heterogeneity. Importantly, it suggests that BrainFIBRE learns representations that capture fundamental, broadly transferable microstructural features, rather than leveraging cohort- or dataset-specific biases.

\subsubsection{Task-Specific Modality Interaction Patterns} To investigate whether and how the interaction patterns between the three brain microstructural measures vary across different datasets and tasks, we visualized the test-set distributions of the five expert weights assigned by the Re-Weighter across four tasks in Fig.\ref{fig:weights}. In the UKB dataset, the synergy expert consistently exhibits the highest average weight, indicating a strong reliance on emergent information fused jointly from all three NODDI maps. Among the uniqueness experts, age prediction relies more heavily on FWF than NDI or ODI, suggesting free water fraction is a more prominent age-related microstructural marker. On the other hand, ODI dominates hippocampal atrophy prediction, suggesting that alterations in neurite dispersion are highly indicative of regional neurodegeneration. 

The Asian dataset exhibits a more uniform expert weight distribution with greater participant variance. While age prediction remains synergy-dominated, it exhibits stronger ODI uniqueness contribution than in UKB, suggesting microstructural alterations in dispersion (potentially reflecting changes in fibre architecture and dendritic organization) provide additional aging-relevant signal in this cohort enriched for cardiometabolic risk. In contrast, the interaction pattern shifts markedly for WMH volume prediction: FWF uniqueness becomes the primary contributor, followed by redundancy and synergy, with minimal reliance on ODI. This pattern suggests that WMH burden is driven predominantly by extracellular water expansion, consistent with established neuropathological findings\cite{ji2017distinct}.

\begin{figure}[t]
    \centering
    \includegraphics[width=1\linewidth]{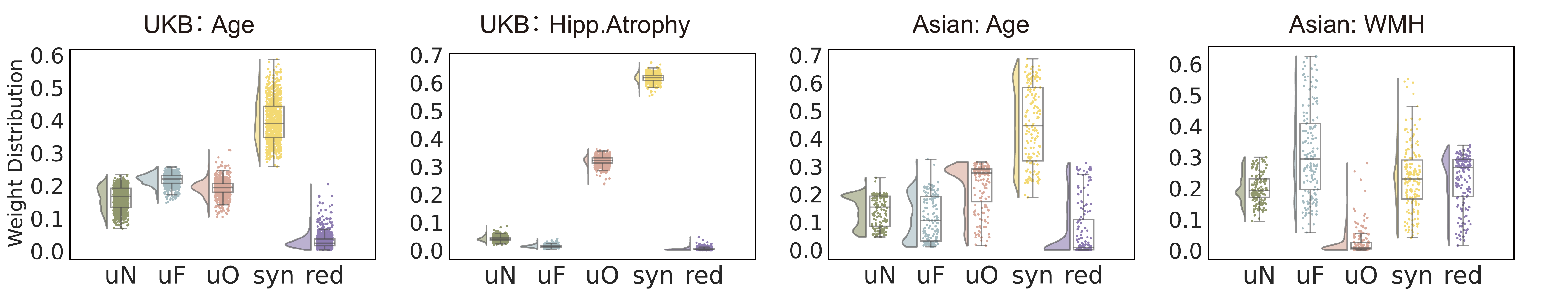}
    \caption{\textbf{Distribution of expert weights across the test set}. The y-axis of each plot illustrates the weight distribution of each of the five experts for a specific task. These include three uniqueness experts (uN, uF, uO), one synergy expert (syn), and one redundancy expert (red).}
    \label{fig:weights}
\end{figure}

\begin{figure}[t]
    \centering
    \includegraphics[width=1\linewidth]{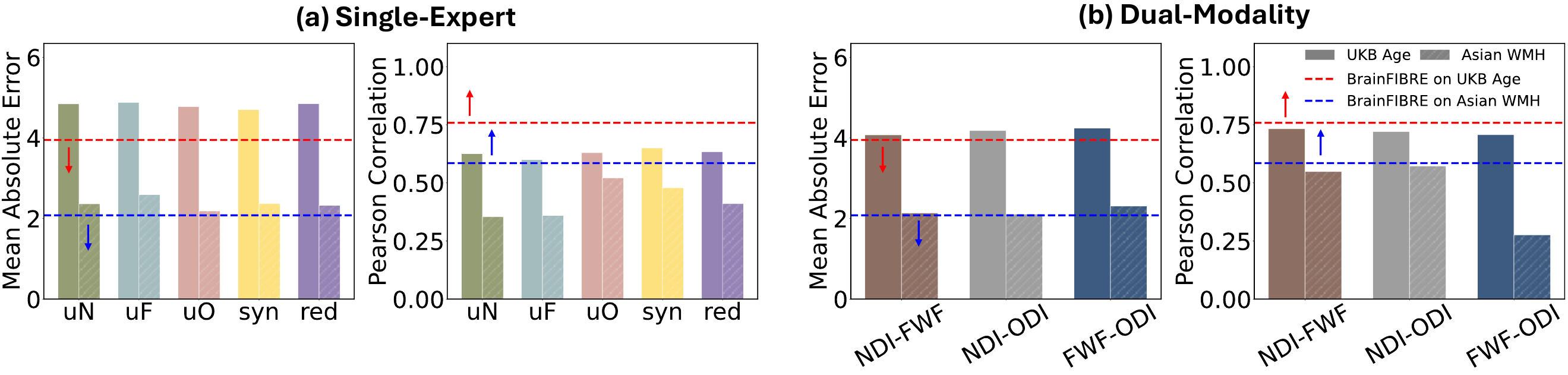}
    \caption{\textbf{Ablation study on age (UKB) and WMH volume (Asian) prediction tasks}. We compared the full BrainFIBRE model (using the final fused embedding) against two settings: (a) Single-Expert (using individual expert embeddings), and (b) Dual-Modality (using fused embeddings derived from only two input NODDI maps).}
    \label{fig:ablation}
\end{figure}

\subsection{Ablation Study}
To investigate the effectiveness of different interaction experts in BrainFIBRE, we conducted an ablation study comparing three settings: (1) \textit{BrainFIBRE}: using the fully fused representation from all experts; (2)\textit{Single-Expert}: evaluating each expert independently, and (3) \textit{Dual-Modality}: omitting one NODDI map to rely on two modality encoders and four corresponding experts (two uniqueness, one synergy, and one redundancy). As shown in Fig.\ref{fig:ablation}, the fused representation consistently outperforms any single expert. This indicates that our model benefits from integrating information from all modality interactions rather than relying on a single type of interaction alone. On the other hand, the dual-modality settings also underperform compared to BrainFIBRE, where dropping any of the modality would degrade the model performance.

Collectively, these results demonstrate that explicitly modeling and aggregating unique, redundant, and synergistic components leads to more robust and generalizable representations across heterogeneous clinical tasks. Additional ablations on individual loss terms are provided in the supplementary materials.


\begin{figure}[t]
    \centering
    \includegraphics[width=1\linewidth]{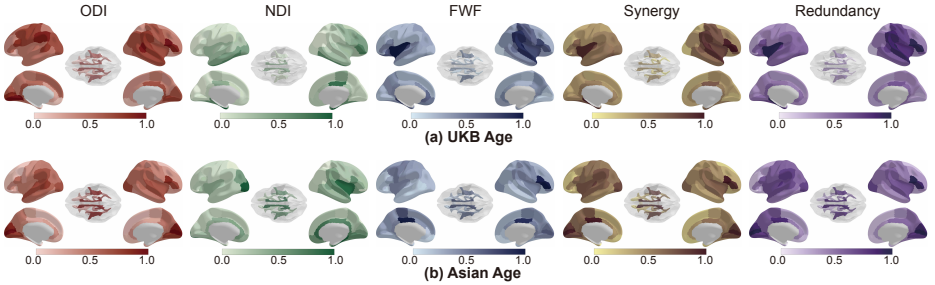}
    \caption{\textbf{Spatial attention patterns captured by different interaction experts}. For both (a) UKB age and (b) Asian age prediction, salient regions are highlighted on both the cortical surface (peripheral views) and white matter tracts (central view).}
    \label{fig:brain3d}
\end{figure}

\subsection{Interpretation}

\subsubsection{Salient Brain Regions Identified by Different Experts} As shown in Fig.\ref{fig:brain3d}, different experts captured distinct spatial patterns across NODDI compartments. For age prediction in UKB, in grey matter (GM), uniqueness experts revealed modality-specific patterns: NDI was localized to the medial temporal cortex, ODI to more posterior regions, and FWF spanned the broadest cortical extent, while synergy concentrated in medial temporal and insular areas. In white matter (WM), age prediction was driven by limbic and association tracts, with FWF showing widespread sensitivity and NDI\&ODI more selective involvement. Synergy and redundancy experts consistently highlighted the cingulum as a hub of multifactorial microstructural aging. These spatial patterns align with prior lifespan diffusion studies reporting medial temporal NDI decline \cite{merluzzi2016age}, widespread free water increases \cite{chad2018re}, and preferential vulnerability of limbic association tracts such as cingulum during healthy aging \cite{peters2014age}.

For the Asian dataset, in GM, uniqueness experts showed a relative shift toward frontal-parietal, midline, and posterior cortical involvement compared with UKB, corresponding to executive control, sensorimotor, and default mode networks, while WM signatures remained largely consistent. This pattern is consistent with vascular-sensitive network vulnerability \cite{ji2023associations, ji2024heart}, potentially reflecting enrichment of cardiometabolic risk factors in the cohort. We provide additional interpretative analyses for other tasks in the supplementary materials.


\section{Conclusion}
We introduce BrainFIBRE, the first foundation model for brain tissue microstructure, pretrained on NODDI-derived maps from UK Biobank. By treating neurite density, orientation dispersion, and free water fraction as distinct modalities, we propose Self-supervised Partial Information Decomposition (SPID) with Counterfactual Candidate Construction (CCC) to train an MoE architecture that explicitly isolates unique, redundant, and synergistic information — without supervision. Evaluations across demographically and ethnically diverse cohorts demonstrates state-of-the-art performance on predicting brain age, cerebrovascular and neurodegenerative markers, and cognitive performance, while the learned representations yield neurobiologically interpretable patterns consistent with established neuropathological findings. These results highlight the value of principled multimodal interaction modeling at the microstructural level and position BrainFIBRE as a versatile foundation for clinical neuroimaging applications.

\section*{Acknowledgments} 

\begin{sloppypar}
This study was supported by the Singapore National Medical Research Council
(NMRC\slash OFLCG19May-0035, NMRC\slash CIRG\slash 1485\slash 2018,
NMRC\slash CSA-SI\slash 0007\slash 2016, NMRC\slash MOH-00707-01,
NMRC\slash CG\slash 435 M009\slash 2017-NUH\slash NUHS, CIRG21nov-0007,
HLCA23Feb-0004, OFIRG24Jul-0049, Human Potential Joint Grant HPJGC25-0002),
A*STAR Singapore (AME Programmatic Fund A20G8b0102, Human Potential Joint
Grant H25P3M00001, HHP Industry-Alignment Fund H24J4a0143), Ministry of
Education (MOE-T2EP20223-0013), and Yong Loo Lin School of Medicine Research
Core Funding, National University of Singapore, Singapore. We thank Christopher Chen, Nagaendran Kandiah, Kwong Hsia Yap, and their teams for participant recruitment and cognitive assessment in SINGER. We also thank Fiona Goh and other lab members for imaging data collection and quality control.

\end{sloppypar}

\newpage

%
%
\bibliographystyle{splncs04}
\bibliography{main}

\clearpage
\appendix

\section{Theoretical Grounding of SPID as a Self-Supervised PID Objective}
\label{sec:spid_pid_grounding}

\setcounter{theorem}{0}

For notational brevity, we write \(O,N,F\) for ODI, NDI, and FWF, respectively.
Classical partial information decomposition (PID) is defined with respect to a target variable.
In our self-supervised setting, the natural target is the latent participant-specific state that generates both augmented views.
We next show that the Counterfactual Candidate Construction (CCC) rules induce five expert-specific sufficient statistics Corresponding to a simplified PID (without explicitly modeling pairwise inter-modal interactions following \cite{xin2025i2moe}) with three unique atoms, one redundant atom, and one synergistic atom.

Let
\[
X^v = (X_O^v, X_N^v, X_F^v), \qquad v \in \{A,B\},
\]
denote the two augmented views of the same participant, and let \(X_{-m}^v\) denote all modalities except modality \(m\).

\begin{definition}[Expert-specific latent factors]
\label{def:expert_factors}
For any perturbed triplet \(\tilde X\), define the five expert factors as follows.

For the three uniqueness experts,
\begin{equation}
C_{uO}(\tilde X)=
\begin{cases}
U_O, & \text{if the \(O\)-modality is present},\\
\bbot, & \text{if the \(O\)-modality is dropped},
\end{cases}
\label{eq:CuO}
\end{equation}
\begin{equation}
C_{uN}(\tilde X)=
\begin{cases}
U_N, & \text{if the \(N\)-modality is present},\\
\bbot, & \text{if the \(N\)-modality is dropped},
\end{cases}
\label{eq:CuN}
\end{equation}
\begin{equation}
C_{uF}(\tilde X)=
\begin{cases}
U_F, & \text{if the \(F\)-modality is present},\\
\bbot, & \text{if the \(F\)-modality is dropped}.
\end{cases}
\label{eq:CuF}
\end{equation}
In a swap perturbation, the Corresponding \(U_m\) is taken from the swapped-in participant.

For redundancy, define
\begin{equation}
C_{\mathrm{red}}(\tilde X)=
\begin{cases}
R, & \text{if all non-dropped modalities in \(\tilde X\)}\\
   & \text{originate from the same participant},\\
\bbot, & \text{otherwise}.
\end{cases}
\label{eq:Cred}
\end{equation}

For synergy, define
\begin{equation}
C_{\mathrm{syn}}(\tilde X)=
\begin{cases}
S, & \text{if \(\tilde X\) contains all three modalities}\\
   & \text{and they are aligned to the same participant},\\
\bbot, & \text{otherwise}.
\end{cases}
\label{eq:Csyn}
\end{equation}
Here, \(\bbot\) denotes that the Corresponding factor is destroyed or undefined.
\end{definition}

For each expert \(k \in \{uO,uN,uF,\mathrm{red},\mathrm{syn}\}\) and participant \(i\), let \(\mathcal{P}_k(i)\) denote the positive candidate set specified by the CCC expert-conditioning rules in Table~1 of the main paper.

\begin{assumption}[Implicit target and PID state]
\label{ass:coarse_pid_state}
Let \(Y\) denote the latent participant state. Assume that
\[
Y = (U_O,U_N,U_F,R,S),
\]
where \(U_O,U_N,U_F,R,S\) are mutually independent latent factors with the following semantics:
\begin{align}
H(U_m \mid X_m^v) &= 0, \qquad
I(U_m; X_{-m}^v) = 0,
\qquad m \in \{O,N,F\},\ v \in \{A,B\}, \label{eq:unique_factor_assumption} \\
H(R \mid X_m^v) &= 0,
\qquad m \in \{O,N,F\},\ v \in \{A,B\}, \label{eq:redundant_factor_assumption} \\
I(S; X_m^v) &= 0, \qquad
H(S \mid X_O^v,X_N^v,X_F^v) = 0,
\qquad m \in \{O,N,F\},\ v \in \{A,B\}. \label{eq:synergy_factor_assumption}
\end{align}
Thus, \(U_m\) is unique to modality \(m\), \(R\) is redundantly present in every modality, and \(S\) is recoverable only from the aligned triplet jointly.

Assume further that the two augmented views are conditionally independent given \(Y\):
\[
X^A \indep X^B \mid Y.
\]
Under this stylized model, the nonzero atoms of the induced five-atom simplified PID of \(Y\) are
\begin{equation}
\UI_O = H(U_O), \quad
\UI_N = H(U_N), \quad
\UI_F = H(U_F), \quad
\RI = H(R), \quad
\SI = H(S).
\end{equation}
\end{assumption}

\begin{assumption}[CCC perturbations]
\label{ass:ccc_perturbations}
A dropout perturbation \(\drop_m\) replaces modality \(m\) with noise independent of \(Y\).
A swap perturbation \(\swap_m\) replaces modality \(m\) by the Corresponding modality from an independent participant \(Y' \sim p(Y)\), where \(Y' \indep Y\).
\end{assumption}

\begin{assumption}[No collisions]
\label{ass:no_collisions}
For each expert \(k \in \{uO,uN,uF,\mathrm{red},\mathrm{syn}\}\), if \(Y\) and \(Y'\) are independent participant states and both Corresponding expert factors are intact (i.e., not equal to \(\bbot\)), then
\[
\mathbb{P}\!\left(
C_k(Y) = C_k(Y')
\,\middle|\,
C_k(Y)\neq \bbot,\ C_k(Y')\neq \bbot
\right)=0.
\]
This rules out degenerate collisions where two unrelated participants have exactly the same expert-specific factor.
\end{assumption}

\begin{theorem}[Correctness of the expert interaction rules]
\label{thm:rule_Correctness}
Under Assumptions~\ref{ass:coarse_pid_state}--\ref{ass:no_collisions}, the positive sets in Table~1 of the main paper are exactly the perturbations that preserve the Corresponding expert factor \(C_k\), and the negative sets are exactly the perturbations that destroy or alter that factor almost surely.

More precisely, for participant \(i\),
\begin{equation}
z \in \mathcal{P}_{uO}(i)
\iff
C_{uO}(z)=U_O^{(i)},
\quad
z \notin \mathcal{P}_{uO}(i)
\implies
C_{uO}(z)\neq U_O^{(i)}
\ \ \text{a.s.},
\label{eq:thm1_uo}
\end{equation}
and analogously for \(uN\) and \(uF\);

\begin{equation}
z \in \mathcal{P}_{\mathrm{red}}(i)
\iff
C_{\mathrm{red}}(z)=R^{(i)},
\quad
z \notin \mathcal{P}_{\mathrm{red}}(i)
\implies
C_{\mathrm{red}}(z)\neq R^{(i)}
\ \ \text{a.s.},
\label{eq:thm1_red}
\end{equation}
and
\begin{equation}
z \in \mathcal{P}_{\mathrm{syn}}(i)
\iff
C_{\mathrm{syn}}(z)=S^{(i)},
\quad
z \notin \mathcal{P}_{\mathrm{syn}}(i)
\implies
C_{\mathrm{syn}}(z)\neq S^{(i)}
\ \ \text{a.s.}
\label{eq:thm1_syn}
\end{equation}
\end{theorem}

\begin{proof}
For the \(uO\) expert, the positive set is
\[
\{b_0,\drop_N,\drop_F,\swap_N,\swap_F\}.
\]
In all these perturbations, the \(O\)-modality still comes from participant \(i\), hence
\[
C_{uO}=U_O^{(i)}.
\]
In contrast, \(\drop_O\) removes the \(O\)-modality, so \(C_{uO}=\bbot\), while \(\swap_O\) replaces it by the \(O\)-modality from an independent participant \(j\), so
\[
C_{uO}=U_O^{(j)} \neq U_O^{(i)}
\quad \text{a.s.}
\]
by Assumption~\ref{ass:no_collisions}. This proves \eqref{eq:thm1_uo}. The arguments for \(uN\) and \(uF\) are identical.

For the redundancy expert, the positive set is
\[
\{b_0,\drop_O,\drop_N,\drop_F\}.
\]
These perturbations preserve participant consistency across all surviving modalities, so the same redundant factor \(R^{(i)}\) remains shared among them, implying
\[
C_{\mathrm{red}}(z)=R^{(i)}.
\]
In contrast, any \(\swap_m\) makes at least one surviving modality originate from a different participant, hence there is no single factor shared across all surviving modalities, and therefore
\[
C_{\mathrm{red}}(z)=\bbot \neq R^{(i)}.
\]
This proves \eqref{eq:thm1_red}.

For the synergy expert, the positive set is
\[
\{b_0\}.
\]
Only \(b_0\) contains all three modalities and keeps them aligned to the same participant, hence only \(b_0\) preserves \(S^{(i)}\). Any \(\drop_m\) removes a necessary modality, and any \(\swap_m\) breaks cross-modal alignment, so in both cases
\[
C_{\mathrm{syn}}(z)=\bbot \neq S^{(i)}.
\]
This proves \eqref{eq:thm1_syn}. Therefore, the CCC expert-conditioning rules preserve exactly the designated expert factors and reject precisely those perturbations that destroy them. 
\end{proof}

\begin{theorem}[The SPID interaction task admits a PID-factorized optimum]
\label{thm:factorized_optimum}
Let $x$ denote the anchor triplet, let \(Q_k=f_k(X)\) be the representation learned by expert \(k\).
Consider the population version of the expert-wise contrastive task induced by the interaction loss in the main paper (Eq.~(8)), i.e., the task of assigning high score to candidates in \(\mathcal{P}_k(i)\) and low score to the other candidates in the CCC pool.
Under Assumptions~\ref{ass:coarse_pid_state}--\ref{ass:no_collisions}, there exists a global optimum of this task of the form
\begin{align}
Q_{uO} &= T_{uO}(U_O), &
Q_{uN} &= T_{uN}(U_N), &
Q_{uF} &= T_{uF}(U_F), \label{eq:factorized_unique} \\
Q_{\mathrm{red}} &= T_{\mathrm{red}}(R), &
Q_{\mathrm{syn}} &= T_{\mathrm{syn}}(S), \label{eq:factorized_shared}
\end{align}
where each \(T_k\) is injective.

Equivalently, each expert can solve its contrastive task using only its designated PID factor.
\end{theorem}

\begin{proof}
Fix an expert \(k\). By Theorem~\ref{thm:rule_Correctness}, for participant \(i\), a candidate is positive if and only if it preserves the same factor \(C_k\) as the anchor, while every negative candidate has a different factor almost surely.

Now choose any injective map \(T_k\) and define
\[
Q_k = T_k(C_k).
\]
Let $x$ denote the anchor Corresponding to participant $i$, and let $z$ denote any candidate from the CCC pool. Because $T_k$ is injective,
\[
Q_k(x)=Q_k(z)
\iff
C_k(x)=C_k(z).
\]
Define a score function
\[
s_k(x,z)=
\begin{cases}
+\alpha, & \text{if } Q_k(x)=Q_k(z),\\
-\alpha, & \text{if } Q_k(x)\neq Q_k(z),
\end{cases}
\qquad \alpha>0.
\]
Under this score, every positive candidate receives strictly larger similarity than every negative candidate. Taking \(\alpha \to \infty\) yields perfect separation between positives and negatives, hence a global optimum of the population contrastive classification problem.

Therefore, expert \(k\) can solve its task using only the factor \(C_k\). Applying the same construction to all five experts yields the factorized solution
\begin{multline*}
(Q_{uO},Q_{uN},Q_{uF},Q_{\mathrm{red}},Q_{\mathrm{syn}})\\
=
(T_{uO}(U_O),T_{uN}(U_N),T_{uF}(U_F),T_{\mathrm{red}}(R),T_{\mathrm{syn}}(S)).
\end{multline*}
This proves the claim. 
\end{proof}

\begin{corollary}[SPID realizes a simplified PID decomposition]
\label{cor:coarse_pid_decomposition}
Suppose training converges to the factorized solution of Theorem~\ref{thm:factorized_optimum}. Let
\[
Q=
(Q_{uO},Q_{uN},Q_{uF},Q_{\mathrm{red}},Q_{\mathrm{syn}}).
\]
Then \(Q\) is an injective transform of
\[
(U_O,U_N,U_F,R,S),
\]
so that
\[
H(Y \mid Q)=0.
\]
Consequently,
\begin{equation}
I(Y;Q)=H(Y)=H(U_O)+H(U_N)+H(U_F)+H(R)+H(S),
\end{equation}
where the five summands are exactly the three unique atoms, one redundant atom, and one synergistic atom of the simplified PID.
\end{corollary}

\begin{proof}
By Theorem~\ref{thm:factorized_optimum}, each expert output is an injective function of one and only one latent factor. Hence the tuple \(Q\) is an injective function of \((U_O,U_N,U_F,R,S)\), and therefore of \(Y\). It follows that
\[
H(Y \mid Q)=0.
\]
Since the factors are mutually independent by Assumption~\ref{ass:coarse_pid_state},
\[
H(Y)=H(U_O)+H(U_N)+H(U_F)+H(R)+H(S).
\]
Therefore,
\[
I(Y;Q)=H(Y)-H(Y \mid Q)=H(Y),
\]
which is precisely the claimed five-term decomposition. 
\end{proof}

\paragraph{Interpretation.}
Theorems~\ref{thm:rule_Correctness} and \ref{thm:factorized_optimum} show that the CCC rules used by SPID are not heuristic.
They define five perturbation-invariant sufficient statistics, and under Assumption~\ref{ass:coarse_pid_state}, these sufficient statistics coincide with the five PID factors.
Therefore, the interaction loss in the main paper is mathematically grounded as a \emph{self-supervised surrogate for a simplified PID decomposition} of an implicit latent participant state.

\begin{remark}[Scope of the claim]
\label{rem:scope}
The mathematically precise claim is that SPID is \emph{PID-grounded at the level of a five-atom coarse decomposition}.
It does not claim exact recovery of the full three-source PID lattice, which is strictly finer than the five experts used here.
Moreover, the present argument is a population-level grounding of the design; it does not imply finite-sample identifiability, nor does it require that optimization necessarily converges to the factorized optimum in practice.
\end{remark}


\section{Empirical Validation of the SPID Disentanglement}
Section~\ref{sec:spid_pid_grounding} has established the theoretical grounding of SPID as a self-supervised PID objective, mathematically linking formal PID quantities (unique, redundant, and synergistic information) to our architectural design. Specifically, Theorem 1 proves that the CCC rules effectively isolate target factors via positive and negative perturbations, and Theorem 2 guarantees the existence of a factorized global optimum where experts solve contrastive tasks solely relying on their intended PID components. 

To empirically validate whether the proposed network architecture learns the PID components, we conduct a post-hoc representation alignment analysis on the held-out UK Biobank (UKB) test set using Canonical Correlation Analysis (CCA). To highlight the relative interaction preferences and mitigate the global baseline Correlation caused by intrinsic anatomical collinearity, the raw CCA scores are global $z$-score normalized across the entire $3 \times 5$ Correlation matrix (Figure~\ref{fig:heeatmap}).

\begin{figure}
    \centering
    \includegraphics[width=1\linewidth]{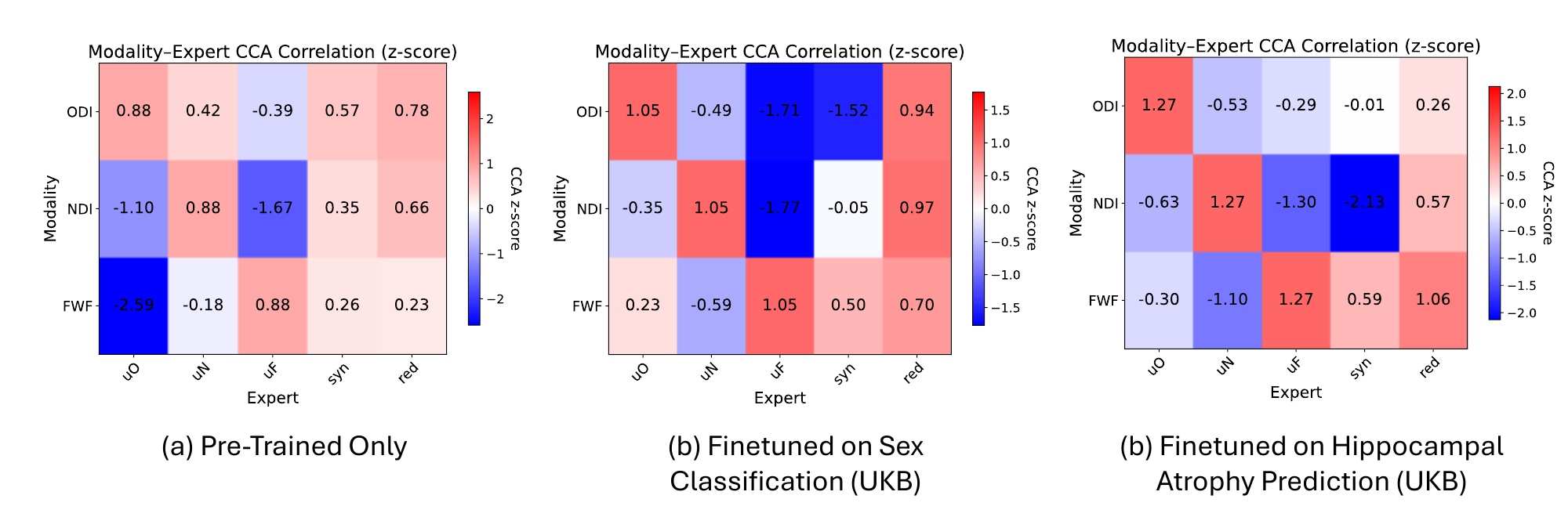}
    \caption{Empirical validation of SPID disentanglement via z-score normalized CCA.}
    \label{fig:heeatmap}
\end{figure}

In this normalized space, the absolute magnitude of the $z$-score represents the significance of deviation from the average baseline alignment. Specifically, a large positive value indicates a prominent Correlation, implying stronger feature alignment. Conversely, a large negative value denotes an exceptionally weak Correlation, indicating feature selection or orthogonalization. A value near $0$ suggests an average alignment with no distinct preference or repulsion.

We extracted the modality embeddings \(h_m\) and expert embeddings \(q_k\) under three different network states: (a) utilizing the frozen pre-trained weights without any downstream finetuning, (b) after finetuning on the Sex Classification task, and (c) after finetuning on the Hippocampal Atrophy Prediction task. As shown in Figure~\ref{fig:heeatmap}, similar representation patterns emerge across all states. Each Uniqueness Expert ($uO, uN, uF$) exhibits the highest positive $z$-score exclusively with its Corresponding modality embedding, while showing much smaller or even negative $z$-scores with the other modalities. For Synergy Expert ($syn$), it exhibits much lower Correlation scores with individual modalities, sometimes dropping to near-zero or strongly negative values (\emph{e.g.}, $-1.52$ and $-2.13$ for ODI and NDI in finetuned tasks). This suggests the emergence of synergistic information: since synergy relies on the modality interactions from the joint combination of all modalities, it should be less aligned with or decoded from any single isolated modality. For Redundancy Expert ($red$), it consistently demonstrates elevated, positive $z$-scores across all three individual modality embeddings. This indicates that $q_{red}$ extracts the shared intersection of information that is simultaneously present in and extractable from every single modality.

\begin{figure}[t]
    \centering
    \includegraphics[width=.8\linewidth]{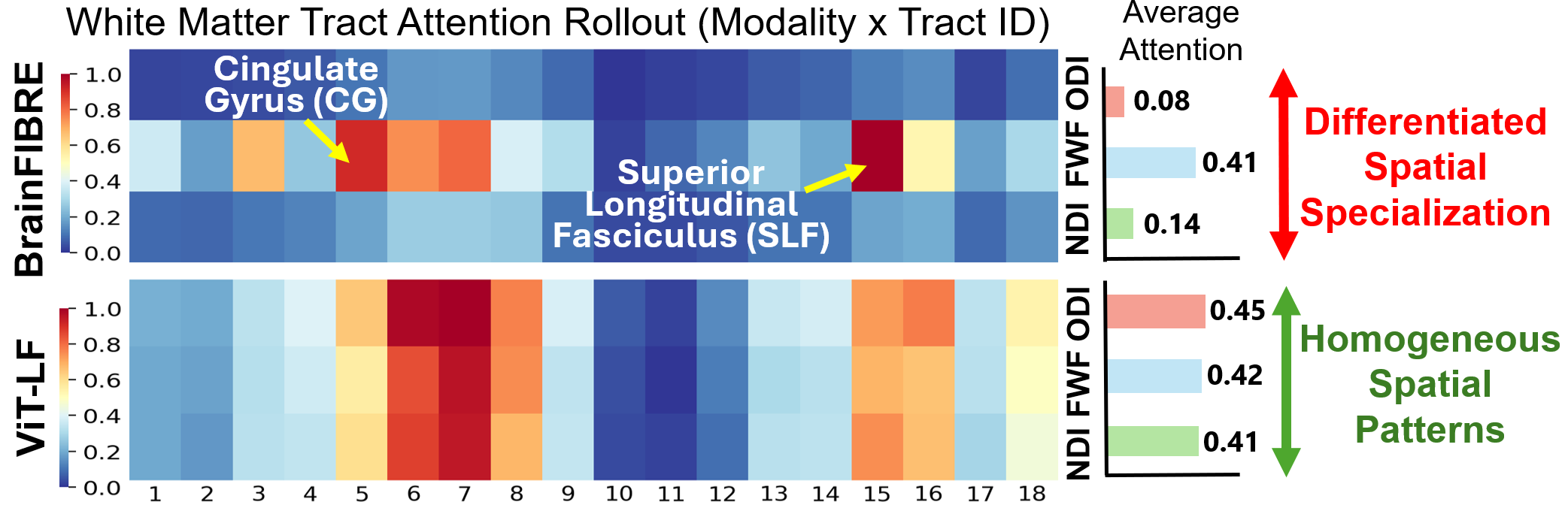}
    \caption{\textbf{White matter tract attention rollout.} BrainFIBRE learns differentiated modality–tract specialization, whereas ViT-LF produces more homogeneous responses.}
    \label{rollout}
\end{figure}

Fig.\ref{rollout} further shows that BrainFIBRE learns disease-relevant modality–tract specialization for WMH prediction. Unlike the late-fusion baseline, which produces relatively homogeneous tract attention, BrainFIBRE selectively emphasizes FWF-related patterns in the cingulate gyrus/cingulum and superior longitudinal fasciculus.


\section{Imaging Acquisition \& Preprocessing}

Diffusion MRI was acquired using multi-shell spin-echo echo-planar imaging protocols optimized for microstructural modeling. In UK Biobank, data were collected on identical 3T Siemens MAGNETOM Skyra scanners using dual diffusion weightings (b = 1000 and 2000 s/mm²) with 50 diffusion directions per shell, and 2 mm isotropic spatial resolution \cite{alfaro2018image}. The Asian cohort was scanned on a 3T Siemens MAGNETOM Prisma system using a dual-shell diffusion scheme (b = 1000 and 2000 s/mm², 50 directions per shell) with 1.8 mm isotropic resolution. The Human Connectome Project in Aging (HCP-A) dataset was acquired on a Siemens MAGNETOM Prisma 3T scanner using high-resolution multi-shell diffusion encoding (b = 1500 and 3000 s/mm², approximately 90 directions per shell) with 1.5 mm isotropic spatial resolution \cite{glasser2013minimal}. Across cohorts, multiple b0 volumes and reversed phase-encoding acquisitions were obtained to enable Correction of susceptibility-induced distortions.

Diffusion-weighted images from all cohorts underwent standardized preprocessing prior to NODDI model fitting. Preprocessing included Correction for susceptibility-induced distortions using reversed phase-encoding b0 images, followed by joint eddy-current and rigid-body motion Correction with slice-wise outlier detection using FSL TOPUP and EDDY \cite{alfaro2018image,glasser2013minimal}. Additional processing steps included gradient nonlinearity Correction (HCP-A), brain extraction, and spatial normalization \cite{ji2023associations}. NODDI parameters were estimated using a consistent fitting framework to generate NDI, ODI, and FWF maps. Quality control combined automated metrics summarizing motion, signal dropout, and model residuals with systematic visual inspection of raw diffusion volumes, Corrected images, and derived parameter maps. Scans failing either automated or visual quality criteria were excluded from downstream analyses. All NODDI maps were aligned to the JHU ICBM-DTI-81 White-Matter Labels 2 mm Atlas for subsequent analyses.


\section{Implementation Details}

\textit{Pre-Training Configurations.} The hyperparameter settings and implementation details for self-supervised pretraining are provided in Table \ref{tab:hyperparameters}. Specifically, we adopted ViT-Small as the backbone for the unimodal encoder. The entire pre-training process utilized eight H200 GPUs, each with 140GB of memory.

\textit{Data Augmentation.} During the pre-training phase, we employ spatial data augmentations for counterfactural candidate construction. We utilized one fixed random seed for each view (A/B) across a given batch. Specifically, we applied random 3D affine transformation with a probability of $p=0.5$. This transformation introduces random rotations within $\pm 10^\circ$, spatial scaling within $\pm 10\%$, and translations of up to $\pm 3$ voxels along all three spatial axes. Any out-of-bounds voxels resulting from the affine transformations are handled using border padding. Furthermore, a random spatial flip is applied along the first spatial axis with a probability of $p=0.2$.

\begin{table}[t]
    \centering
    \caption{Implementation details and hyperparameters for self-supervised pretraining.}
    \label{tab:hyperparameters}
    
    \fontsize{7pt}{9pt}\selectfont
    \renewcommand{\arraystretch}{1}
    \setlength{\tabcolsep}{8pt} 
    
    \begin{tabular}{ll}
        \toprule
        \textbf{Configuration} & \textbf{Value} \\
        \midrule
        
        \multicolumn{2}{l}{\textit{Architecture \& Input}} \\
        Input Size  & $96 \times 112 \times 96$\\
        Patch Size & $16 \times 16 \times 16$ \\
        Encoder Embedding Dim (\(d_{enc}\)) & 384  \\
        Expert Embedding Dim (\(d_{exp}\)) & 128 \\
        
        \midrule
        \multicolumn{2}{l}{\textit{Optimization}} \\
        Training Epochs (Warmup) & 200 (5) \\
        Batch Size & 280 \\
        Learning Rate & $2 \times 10^{-4}$ \\
        Weight Decay & $1 \times 10^{-4}$ \\
        Optimizer Betas ($\beta_1, \beta_2$) & 0.9, 0.999 \\
        Gradient Clipping & 1.0 \\
        
        \midrule
        \multicolumn{2}{l}{\textit{Loss Function Parameters}} \\
        Temperature ($\tau_{cont}$, $\tau_{inter}$) & 0.2, 0.12 \\
        Weights ($\beta_{balance}$, $\gamma_{entropy}$) & 0.05, 0.01 \\
        
        \midrule
        \multicolumn{2}{l}{\textit{Data Augmentation}} \\
        Random Flip & Spatial Axis 0 \\
        Random Affine & Rotation $\pm10^{\circ}$  \\
         & Scale $\pm10\%$ \\
         & Translation $\pm3$voxels \\
        
        \bottomrule
    \end{tabular}
\end{table}


\section{Baselines Comparison and Ablation Studies}

\subsection{Additional Multimodal Baselines}
To provide a comprehensive comparison, we implemented two additional multimodal Vision Transformer (ViT) baselines and compared them on the UKB age prediction and Asian WMH prediction tasks. These baselines are explicitly designed to utilize the exact same three NODDI maps and ViT backbones as BrainFIBRE, but with standard fusion strategies. This ensures that any observed superiority of our model can be strictly attributed to the proposed SPID pre-training paradigm and PID-inspired MoE framework, rather than merely the advantage of having rich multimodal inputs.

\begin{itemize}
    \item \textit{ViT-EF (Early Fusion)}: A 3-channel ViT-S architecture that accepts the channel-wise concatenation of the three NODDI maps as a single input.

    \item \textit{ViT-LF (Late Fusion)}: An architecture utilizing three independent ViT-S encoders to process each NODDI map separately, followed by embedding concatenation for modality fusion
\end{itemize}

To ensure a fair comparison, both baselines were assessed under two training paradigms: (1) self-supervised contrastive pre-training (without interactions) on the identical UKB dataset used for BrainFIBRE followed by downstream fine-tuning, and (2) fully supervised training from scratch directly on the target labeled datasets.

As reported in Table~\ref{tab:baselines}, BrainFIBRE consistently outperforms both standard fusion baselines across the UKB age prediction and Asian WMH prediction tasks. Although applying self-supervised pre-training to these baselines improves their efficacy over the supervised-from-scratch variants, their optimal performance remains significantly lower than that of BrainFIBRE.

\subsection{Ablation Studies on Loss Terms}
We conducted an ablation study to isolate the contribution of each individual loss term within our proposed SPID pre-training objective. Specifically, we pre-trained three variants of BrainFIBRE by systematically removing one loss component at a time: (1) w/o $\mathcal{L}_{Cont}$, which removes the global contrastive loss; (2) w/o $\mathcal{L}_{entropy}$, which drops the entropy regularization term; and (3) w/o $\mathcal{L}_{balance}$, which excludes the expert balancing loss. These variants were then fine-tuned and evaluated on the downstream tasks to quantify the necessity of each optimization constraint (Table~\ref{tab:ablation}).  We also conduct ablation on Synergy and/or Redundancy experts and interaction loss $\mathcal{L}_{Inter}$ (Table~\ref{tab:ablation_add}). Removing either part degrades performance, confirming their individual contributions.

\begin{table}[t]
    \centering
    \begin{threeparttable}
        \caption{Additional baselines comparisons on UKB and Asian datasets.}
        \label{tab:baselines}
        
        \fontsize{6.5pt}{7.5pt}\selectfont 
        \renewcommand{\arraystretch}{1.2}
        \setlength{\tabcolsep}{4pt}
        
        \begin{tabular}{l c c c c c}
            \toprule
            \multirow{2}{*}{\textbf{Model}} & 
            \multirow{2}{*}{\textbf{Pretrain}} & 
            \multicolumn{2}{c}{\textbf{UKB - Age}} & 
            \multicolumn{2}{c}{\textbf{Asian - WMH}} \\
            
            \cmidrule(lr){3-4} \cmidrule(lr){5-6}
            
            & & MAE & Corr & MAE & Corr \\
            \midrule
            
            \multirow{2}{*}{ViT-EF} & \xmark & 4.23\std{0.01} & 0.71\std{0.01} & 2.35\std{0.32} & 0.32\std{0.08} \\
            & \cmark & 4.15\std{0.19} & 0.72\std{0.04} & 2.45\std{0.28} & 0.38\std{0.02} \\
            
            \midrule
            
            \multirow{2}{*}{ViT-LF} & \xmark & 4.39\std{0.02} & 0.69\std{0.01} & 2.43\std{0.29} & 0.27\std{0.05} \\
            & \cmark & 4.19\std{0.02} & 0.71\std{0.01} & 2.31\std{0.38} &0.35\std{0.04} \\
            
            \midrule
            
            \textbf{BrainFIBRE} & \cmark & \textbf{3.95}\std{0.02} & \textbf{0.76}\std{0.00} & \textbf{2.08}\std{0.30} & \textbf{0.59}\std{0.02} \\
            
            \bottomrule
        \end{tabular}

        \begin{tablenotes}
            \tiny
            \item[] \hspace{-1.6em} ViT-EF: Early Fusion. ViT-LF: Late Fusion.
        \end{tablenotes}

    \end{threeparttable}
\end{table}

\begin{table}[t]
    \centering
    \begin{threeparttable}
        \caption{Ablation study of different loss components in BrainFIBRE.}
        \label{tab:ablation}
        
        \fontsize{6.5pt}{7.5pt}\selectfont 
        \renewcommand{\arraystretch}{1.2}
        \setlength{\tabcolsep}{4.5pt} 
        
        \begin{tabular}{c c c c c c c c}
            \toprule
            \multicolumn{4}{c}{\textbf{Loss}} & 
            \multicolumn{2}{c}{\textbf{UKB - Age}} & 
            \multicolumn{2}{c}{\textbf{Asian - WMH}} \\
            
            \cmidrule(lr){1-4} \cmidrule(lr){5-6} \cmidrule(lr){7-8}
            
            $\mathcal{L}_{Inter}$ & $\mathcal{L}_{Cont}$ & $\mathcal{L}_{balance}$ & $\mathcal{L}_{entropy}$ & MAE & Corr & MAE & Corr \\
            \midrule
            
            \cmark &        & \cmark & \cmark & 4.24\std{0.10} & 0.70\std{0.01}& 2.19\std{0.20} & 0.47\std{0.09}\\
            \cmark & \cmark &        & \cmark & 4.16\std{0.08} & 0.72\std{0.01} & 2.53\std{0.28} & 0.34\std{0.02} \\
            \cmark & \cmark & \cmark &        & 4.17\std{0.04}& 0.71\std{0.00}& 2.17\std{0.31}& 0.46\std{0.02}\\
            
            \cmark & \cmark & \cmark & \cmark & \textbf{3.95}\std{0.02} & \textbf{0.76}\std{0.00} & \textbf{2.08}\std{0.30} & \textbf{0.59}\std{0.02} \\
            
            \bottomrule
        \end{tabular}

        \begin{tablenotes}
            \tiny
            \item[] \hspace{-1.6em} ViT-EF: Early Fusion. ViT-LF: Late Fusion.
        \end{tablenotes}
        
    \end{threeparttable}
\end{table}

\begin{table}[htbp]
    \centering
    \caption{Ablation on Synergy and/or Redundancy experts and interaction loss $\mathcal{L}_{Inter}$}
    \label{tab:ablation_add}
    \resizebox{0.75\linewidth}{!}{%
        \begin{tabular}{lcccccc}
            \toprule
            & \multicolumn{1}{c}{\textbf{Ours}}
            & \multicolumn{3}{c}{\textbf{Expert Ablation}}
            & \multicolumn{1}{c}{\textbf{Loss Ablation}} \\
            \cmidrule(lr){2-2} \cmidrule(lr){3-5} \cmidrule(lr){6-6}
            \textbf{Metric} & BrainFIBRE & w/o $u_\text{syn}$ & w/o $u_\text{red}$ & w/o both & $-\mathcal{L}_\text{int}$ \\
            \midrule
            Corr. & \textbf{0.76}\std{0.00} & 0.66\std{0.02} & 0.68\std{0.02} & 0.67\std{0.01} & 0.71\std{0.00} \\
            MAE   & \textbf{3.95}\std{0.02} & 4.55\std{0.09} & 4.41\std{0.08} & 4.48\std{0.05} & 4.22\std{0.07} \\
            \bottomrule
        \end{tabular}%
    }
\end{table}


\section{Visualization and Interpretation}

\subsection{Task-Specific Modality Interaction Patterns} We present the task-specific modality interaction patterns for the HCP-Aging dataset. To quantitatively evaluate these interactions, we plot the distribution of expert weights across the test set for both the Flanker and Cardsort prediction tasks (Figure~\ref{fig:expert_weights}).

For the HCP-Aging Flanker prediction task, the expert weight distribution indicates that the free water fraction (FWF) provides the largest independent contribution, followed by neurite density (NDI) and orientation dispersion (ODI). In contrast, synergistic information contributes moderately, and redundancy remains relatively limited. This pattern suggests that extracellular water changes capture a substantial portion of the task-relevant signal, while neurite density and orientation provide complementary microstructural information

In the HCP-Aging Card Sort task, the weight distribution shows that synergistic information accounts for the largest proportion of the predictive signal, followed by redundancy, while the unique contributions from NDI, FWF, and ODI are comparatively minor. This suggests that performance on this task is primarily supported by integrated microstructural information across diffusion compartments rather than by any single microstructural feature.

\begin{figure}[t]
    \centering
    \includegraphics[width=1\linewidth]{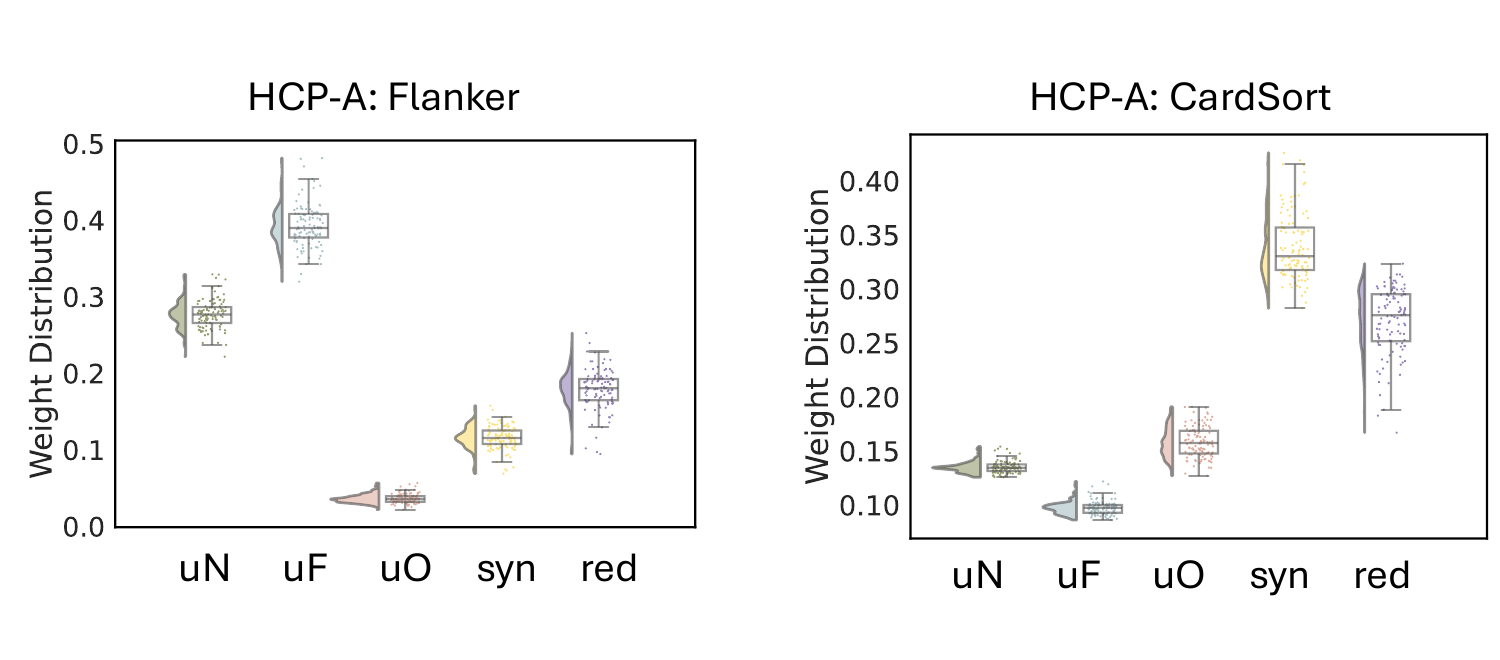}
    \caption{Distribution of expert weights across the test set.}
    \label{fig:expert_weights}
\end{figure}

\begin{figure}[t]
    \centering
    \includegraphics[width=1\linewidth]{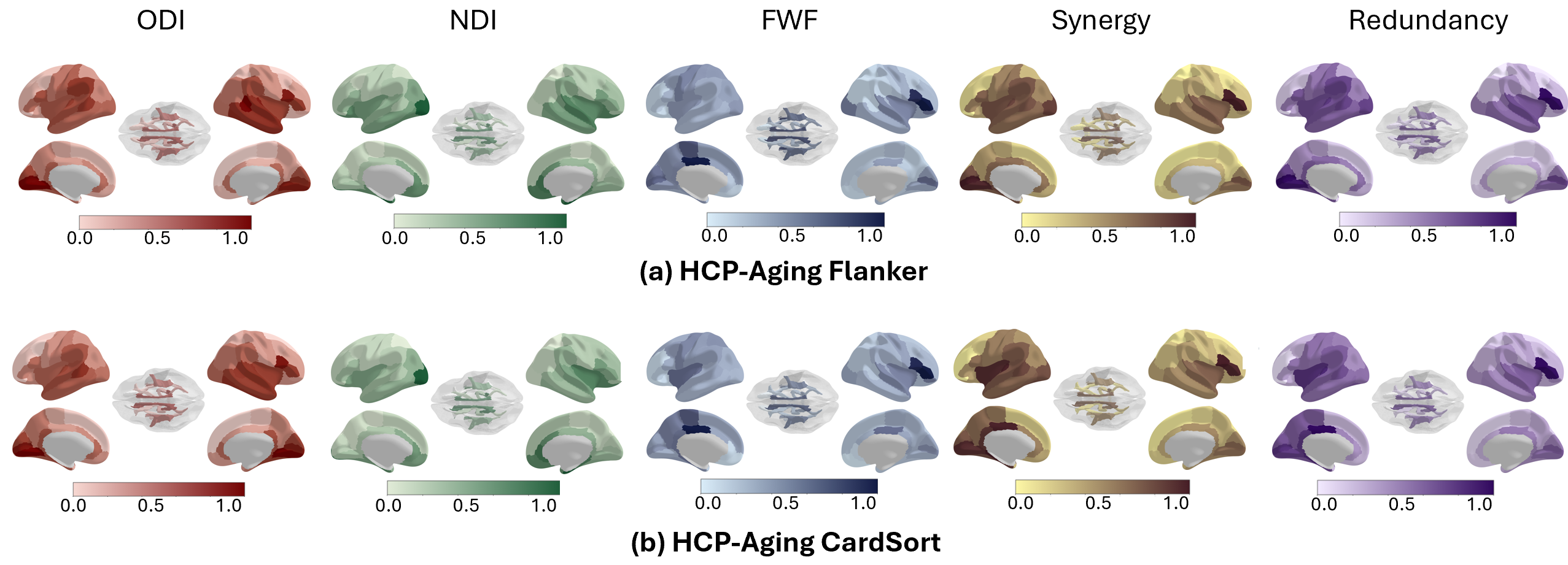}
    \caption{Spatial attention patterns captured by different interaction experts.}
    \label{fig:placeholder}
\end{figure}

\subsection{Salient Brain Regions Identified by Different Experts} In Section 4.5 of the main manuscript, we provided two examples of salient brain regions identified by different experts for the UKB and Asian age prediction tasks. Here, we present additional visualizations for the Flanker and Cardsort prediction tasks using the HCP-Aging dataset.

Spatial interpretations of the expert representations reveal that while the Flanker and Card Sort tasks rely on shared white matter infrastructural substrates, they exhibit distinct, task-specific microstructural dependencies at the cortical level. 

In white matter, performance on both tasks critically depends on the combined contributions of FWF, NDI, and ODI within long-range association pathways—specifically the cingulum bundle, superior longitudinal fasciculus, and uncinate fasciculus. These tracts structurally link frontal executive regions with parieto-temporal cortices, facilitating the distributed frontoparietal and limbic communication essential for overarching executive attention and cognitive flexibility. 

At the cortical level, NDI consistently highlights association and limbic regions (\emph{e.g.}, entorhinal cortex, temporal pole) crucial for cognitive integration across both paradigms. However, the tasks diverge significantly in their specific spatial sensitivities to FWF and ODI. For the Flanker task, which probes inhibitory control and selective attention, FWF uniqueness heavily dominates middle frontal and sensory (occipital/perisylvian) cortices, suggesting that extracellular properties strongly influence attentional filtering and sensory transmission. Conversely, during the Card Sort task—which assesses task switching—FWF contributions shift towards the cingulo-opercular system (posterior cingulate, insular, and inferior frontal regions), while ODI captures the variations in visual and perisylvian sensory regions. Together, these patterns demonstrate how distinct microstructural features dynamically support specialized executive demands atop a shared anatomical network.

\end{document}